\pdfoutput=1

\documentclass[11pt]{article}

\usepackage[final]{acl}

\usepackage{times}
\usepackage{latexsym}

\usepackage[T1]{fontenc}

\usepackage[utf8]{inputenc}

\usepackage{microtype}

\usepackage{inconsolata}

\usepackage{graphicx}

\usepackage{amsthm}
\usepackage{amsmath}
\usepackage{amssymb}
\usepackage{mathtools}

\usepackage{algpseudocode}
\usepackage{amsmath}

\usepackage[utf8]{inputenc} 
\usepackage[T1]{fontenc}    
\usepackage{hyperref}       
\usepackage{url}            
\usepackage{booktabs}       
\usepackage{amsfonts}       
\usepackage{nicefrac}       
\usepackage{microtype}      

\usepackage{bm}

\usepackage{verbatim}
\usepackage{balance}
\usepackage{graphicx}
\usepackage{multirow}
\usepackage{xspace}
\usepackage{threeparttable}
\usepackage{enumitem}
\usepackage{makecell}
\usepackage{bbding}
\usepackage{subfig}
\usepackage{float}
\usepackage{wrapfig}
\usepackage[export]{adjustbox}
\usepackage{cleveref}
\usepackage{tikz}
\usepackage{pifont}

\usepackage{algorithm}
\usepackage{algpseudocodex}

\newtheorem{theorem}{Theorem}

\usepackage{amsmath,amsfonts,bm}

\def\eqref#1{equation~\ref{#1}}

\def\1{\bm{1}}

\def\vx{{\bm{x}}}
\def\vy{{\bm{y}}}

\DeclareMathAlphabet{\mathsfit}{\encodingdefault}{\sfdefault}{m}{sl}
\SetMathAlphabet{\mathsfit}{bold}{\encodingdefault}{\sfdefault}{bx}{n}

\newcommand*{\para}[1]{\noindent\textbf{#1}}

\newcommand{\vpara}[1]{\vspace{0.04in}\noindent\textbf{#1}\xspace}

\newcommand{\model}{TreeRL\xspace}
\newcommand{\treemodel}{EPTree\xspace}
\newcommand{\hide}[1]{}

\title{\model: LLM Reinforcement Learning with On-Policy Tree Search}

\author{
  Zhenyu Hou\textsuperscript{*,1} \quad
  Ziniu Hu\textsuperscript{*,2} \quad
  Yujiang Li\textsuperscript{*,1} \quad
  Rui Lu\textsuperscript{*,1} \quad
  Jie Tang\textsuperscript{1} \quad  Yuxiao Dong\textsuperscript{1} 
  \\
      \textsuperscript{1}Tsinghua University \quad
    \textsuperscript{2}California Institute of Technology 
  }

\newcommand{\ZN}[1]{\textcolor{red}{[Ziniu: #1]}}

\begin{document}
\maketitle

\renewcommand{\thefootnote}{\fnsymbol{footnote}}
    \footnotetext[1]{Equal contribution; order is alphabetical.}
\renewcommand{\thefootnote}{\arabic{footnote}}

\begin{abstract}
 Reinforcement learning (RL) with tree search has demonstrated superior performance in traditional reasoning tasks. Compared to conventional independent chain sampling strategies with outcome supervision, tree search enables better exploration of the reasoning space and provides dense, on-policy process rewards during RL training but remains under-explored in On-Policy LLM RL.
 We propose \model, a reinforcement learning framework that directly incorporates on-policy tree search for RL training. Our approach includes intermediate supervision and eliminates the need for separate reward model training. Existing approaches typically train a separate process reward model, which can suffer from distribution mismatch and reward hacking. 
 We also introduce a cost-effective tree search approach that achieves higher search efficiency under the same generation token budget by strategically branching from high-uncertainty intermediate steps rather than using random branching. 
 Experiments on challenging math and code reasoning benchmarks demonstrate that \model achieves superior performance compared to traditional ChainRL, highlighting the potential of tree search for LLM. 
 \model is open-sourced at \url{https://github.com/THUDM/TreeRL}.

\end{abstract}

\section{Introduction}
Large language models (LLMs) have demonstrated remarkable capabilities across diverse complex reasoning tasks~\citep{achiam2023gpt,team2023gemini,dubey2024llama3}, including mathematics~\cite{shao2024deepseekmath}, programming~\citep{lozhkov2024starcoder,zhu2024deepseek}, and autonomous agents~\citep{zhouwebarena}.
Reinforcement learning (RL) has emerged as a powerful approach to significantly improve the reasoning abilities of LLMs by optimizing the policy through reward feedback~\citep{openaio1,guo2025deepseek,hou2025advancing,shao2024deepseekmath}.

\begin{figure}[t]
    \centering
    \begin{minipage}{0.235\textwidth}
        \centering
        \includegraphics[width=\textwidth]{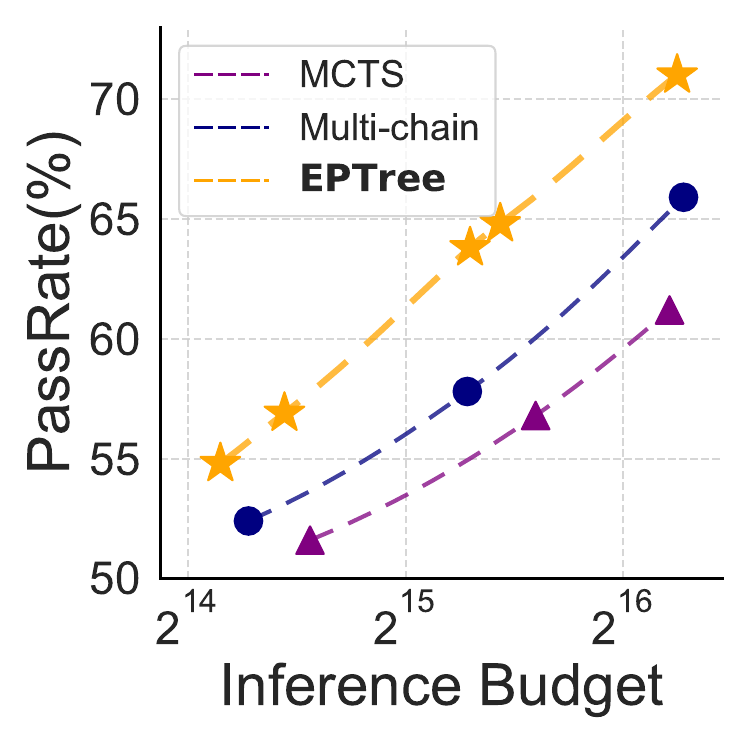}
        \label{fig:third_image}
    \vspace{-7mm}
    \end{minipage}
    \hfill
    \begin{minipage}{0.235\textwidth}
        \centering
        \includegraphics[width=\textwidth]{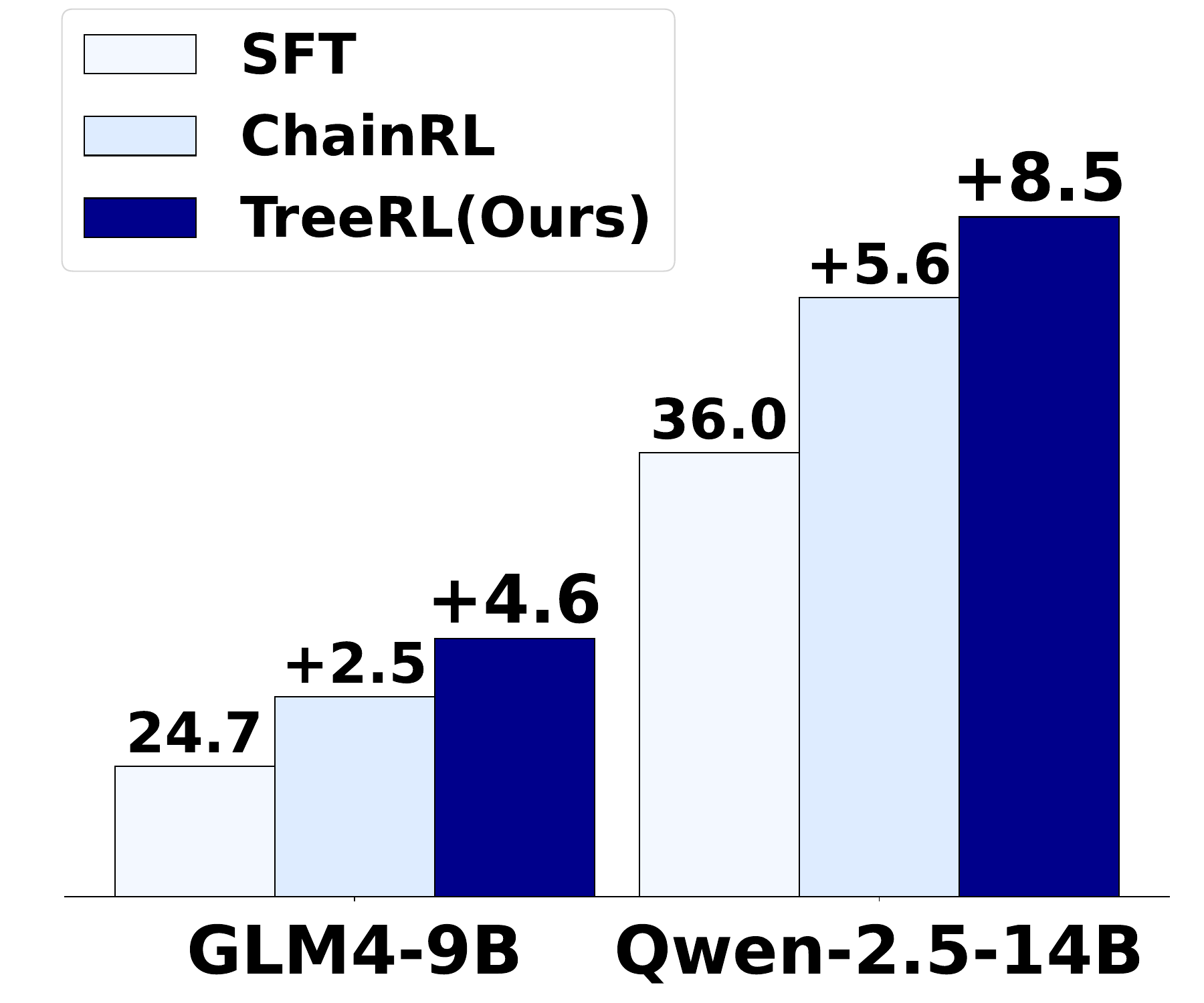}
        \label{fig:third_image}
    \vspace{-3mm}
    \end{minipage}
    \vspace{-3mm}
    \caption{\textit{Left}: Performance comparison of sampling strategies. \treemodel consistently outperforms i.i.d multi-chain sampling and MCTS under different inference budgets. \textit{Right}: \model powered with \treemodel demonstrates better performance than ChainRL with i.i.d multi-chain sampling.}
    \label{fig:performance_preview}
\end{figure}

\begin{figure*}
    \centering
    \includegraphics[width=\linewidth]{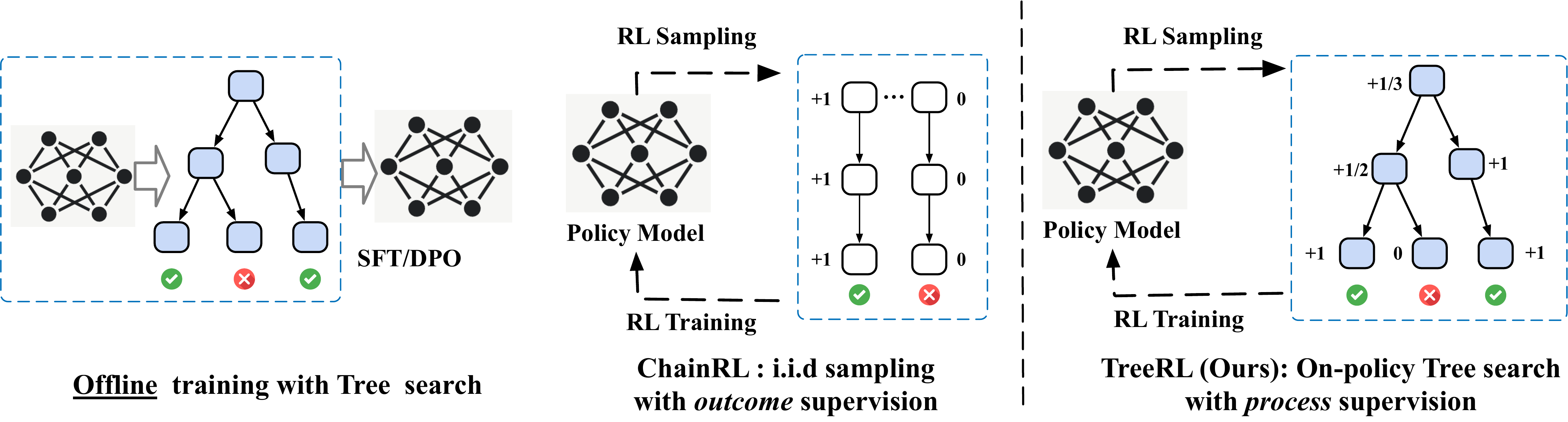}
    \caption{Illustration of offline training with tree search (Left), traditional ChainRL with online i.i.d multi-response sampling (Middle), and \model with tree search (Right).}
    \label{fig:illustration}
\end{figure*}


Current RL methods for LLM training generally independently sample multiple trajectories~\cite{shao2024deepseekmath,wang2024mathshepherd,touvron2023llama2} and obtain reward signals based on the final answers. 
However, tree search, which has demonstrated significant success in other domains like AlphaZero~\cite{silver2017mastering}, remains under-developed in reinforcement learning for LLM reasoning. 
Existing efforts have mainly focused on using tree search to enhance inference-time performance alongside an external reward~\cite{zhang2024rest, chen2024alphamath}, or to produce data for offline training~\cite{chen2024alphamath,xie2024monte,zhang2024rest} (e.g., finetuning or DPO~\cite{rafailov2023direct}), as illustrated in Figure~\ref{fig:illustration}. But \citet{guo2025deepseek} also demonstrates the limitation of distribution shift and reward hacking in offline tree search compared to online RL training. Up to now, the potential of on-policy RL training incorporating tree search to improve LLM reasoning remains largely unexplored.

\hide{
The challenges are two-folds. First, classical Monte Carlo Tree Search (MCTS) can be less effective than independently sampling multiple responses under the same inference cost, which could hinder the performance of reinforcement learning. As shown in Table~\ref{fig:illustration}, MCTS exhibits worse PassRate performance and thus is less powerful in searching for the correct answer. 
Second, MCTS is far less efficient than sampling multiple responses because it demands numerous iterative, step-by-step generation, which suffers from lower parallel and is not friendly for the current LLM inference engine. 
}

\hide{
The challenges are two-fold. First, classical Monte Carlo Tree Search (MCTS)~\citep{browne2012mctssurvey} can be less effective than independently sampling multiple responses at the same inference cost, which can hinder the performance of reinforcement learning. As shown in Figure~\ref{fig:performance_preview}, MCTS demonstrates poorer PassRate performance, making it less effective in searching for the correct answer.
Second, MCTS is significantly less efficient than sampling multiple responses, as it requires numerous iterative, step-by-step generations. This process suffers from lower parallelism and is not well-suited for modern LLM inference engines.
}

The challenges are two-fold. First, classical Monte Carlo Tree Search (MCTS)~\citep{browne2012mctssurvey} can be less effective nd efficient than independently sampling multiple responses. MCTS achieves a lower PassRate performance under the same inference cost, as shown in Figure~\ref{fig:performance_preview}. The step-by-step generation requires numerous iterations and is not friendly to modern LLM inference engines.
Second, though tree search could provide fine-grained process supervision, the derived offline process reward model almost contributes to no performance improvement in RL training~\cite{guo2025deepseek}. 

To address this gap, we propose \model, a reinforcement learning (RL) approach for LLMs employed with tree search. Under the same inference cost as independent multiple sampling, our method generates more diverse and effective responses and also provides on-policy process supervision signals to further boost RL performance.

First, we introduce an efficient and effective tree search strategy \treemodel.
Unlike MCTS which breaks answers into smaller parts to allow the model to explore step-by-step, we obtain a new response by forking branches from the most uncertain intermediate tokens in the existing tree based on entropy and continuing the generation until the final answer. Thus \treemodel requires fewer tokens but encourages effective exploration for the diverse answers.
And it typically requires only around two iterations to form the generation trees.



One step further, we leverage the tree search for reinforcement learning with process supervision. Each step in the trees is assigned a credit based on advantage, i.e., how much that step improves the likelihood of reaching a correct solution compared to other steps. The process signal for a given reasoning step is calculated as a weighted sum of two components: 1) global advantage, which reflects the step's potential over the overall correctness rate 
for the question, and 2) local advantage, which quantifies the improvement the step provides compared to its parent step in the tree. Since these advantage signals are derived directly from the on-policy generated trees, they are inherently resistant to reward hacking~\citep{skalse2022defining} and do not rely on any additional reward models.


We evaluate \model on challenging college-level and competition-level math and code reasoning benchmarks based on Qwen~\citep{qwen2_5} and GLM~\citep{glm2024chatglm}. Experiments show that \model achieves superior performance and demonstrates advantages over traditional independent multi-chain sampling. The gain benefits from both \treemodel with promising PassRate performance and process supervision. These results highlight the potential of RL with tree search to advance complex reasoning capabilities for LLM. The implementation of \model is available at \url{https://github.com/THUDM/TreeRL}.

\begin{figure*}[!h]
    \centering
    \includegraphics[width=1.0\linewidth]{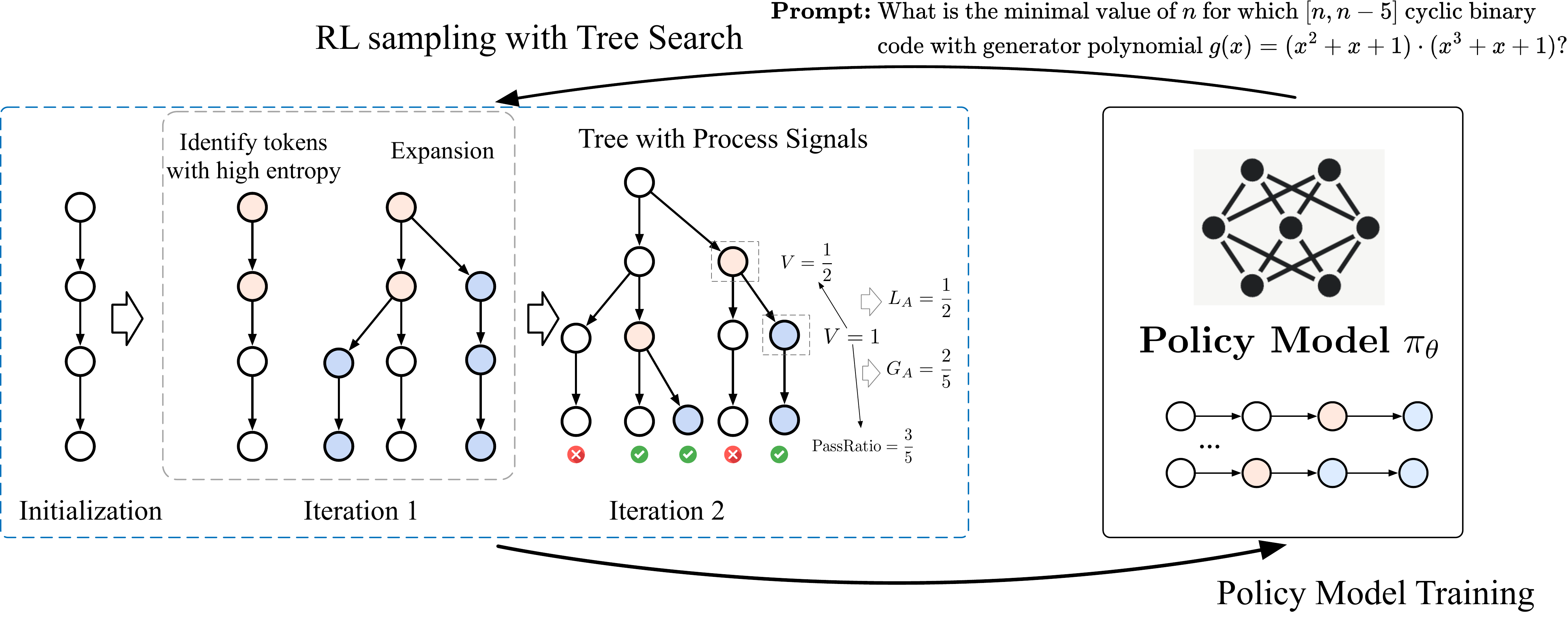}
    \caption{Illustration of \model. In each iteration, \model first performs a tree search using \treemodel, progressively expanding branches from the top-$N$ most uncertain tokens. The resulting tree, along with the process supervision signals derived from each step, is then fed into reinforcement learning to update the policy model.}
    \label{fig:overview}
\end{figure*}

\section{Preliminary}

To align the fine-tuned model \(\pi_\theta\) with feedback signals, \citet{ouyang2022training} proposes to apply reinforcement learning (RL) to enable LLM to learn from self-exploration.
Reinforcement learning maximizes a reward signal, e.g., human preference or final answer correctness.
The typical RL for the LLM process works as follows: for a given prompt \(\vx\), the policy model \(\pi_\theta\) generates \(K\) possible responses, denoted as \((\vy_1, \dots, \vy_K)\). The reward function $r(\vx, \vy_i)$ then assigns a scalar reward to each pair \((\vx, \vy_i)\). Afterward, the policy model \(\pi_\theta\) is updated using reinforcement learning to optimize the following objective:
\begin{equation}
    \underset{\vx \sim p_{\text{data}}, \vy \sim \pi_\theta}{\mathbb{E}} \frac{1}{K} \sum_i^{K}  A(\vx, \vy_i) \log \pi_{\theta}(\vy_i|\vx)
    \label{eq:rl_target}
\end{equation}
where $A(\cdot)$ is the advantage function and usually defined as $A(\vx, \vy_i) = \beta(r(\vx, \vy_i) - b)$, where $b$ is the baseline~\cite{mei2022role,chung2021beyond} and varies in different methods. For example, $A(\vy_i) = \frac{r(\vy_i) - \text{mean}\{r(\vy_i)\}}{\text{std}\{r(\vy_i)\}}$ in GRPO~\cite{shao2024deepseekmath} and $A(\vy_i) = r(\vy_i) - \frac{1}{K-1}\sum_{j\neq i}^K r(\vy_i)$ in RLOO~\cite{ahmadian2024rloo}.

\section{\model: Reinforcement Learning with Tree Search}
In this section, we present \model to improve LLM reasoning with tree search. Figure~\ref{fig:overview} shows the overview of \model. We first present an efficient and effective tree search algorithm \treemodel, which guides tree search by token-level uncertainty instead of traditional Monte Carlo Tree Search (MCTS)~\citep{browne2012mctssurvey}. Then, we show how to integrate the tree search into reinforcement learning with process supervision to improve reasoning capability further.

\subsection{Entropy-Guided Tree Search}
\label{sec:entropy-tree}
We aim to optimize the tree-search algorithm for RL training and thus emphasize the PassRate metric, which evaluates the algorithm's ability to generate diverse yet correct answers under a given inference budget. Furthermore, the tree-search algorithm must be efficient and highly parallelizable. In contrast, traditional MCTS requires numerous iterative generations, making it less efficient in current LLM inference engines like VLLM~\citep{kwon2023efficient}.

We propose an entropy-guided tree search algorithm, \treemodel.
The core idea is to iteratively expand the search tree by forking new branches (nodes) from the top-$N$ most uncertain tokens (as measured by entropy) across existing trees.
This encourages exploration in regions of high model uncertainty, leading to improved performance.  
Critically, \treemodel can generate multiple trees in parallel, and the expansion process requires only around 2 iterations to build a diverse and informative tree, making it highly efficient.
The algorithm runs the following steps.

\vpara{Initialization.} To generate $M$ trees in parallel, we first construct \( M \) chains by generating $M$ responses for a given prompt \( \vx \) as initialization of $\mathcal{T}_i$ for further expansion:
\[
Y^{(i)} = \{\vy_i \sim \pi_\theta(\cdot \mid \vx)\},~ \text{for} ~ i = 1, 2, \ldots, M
\]
where \( \pi_\theta \) is the policy model and $\vx$ is the prompt. 



\vpara{Forking token selection.} Next, we aim to expand the trees by forking new branches from the existing trees. We propose forking from the tokens with the highest uncertainty, as these tokens provide the most informative signals for the model and encourage exploration.
We use cross-entropy as a measure of uncertainty, which quantifies the uncertainty in the policy model \(\pi_\theta\) when predicting a given token. To promote expansion, the top-\(N\) tokens with the highest entropy values are selected across the whole tree $\mathcal{T}_i$. 
Specifically, the entropy of each token \(v\) in the tree \(\mathcal{T}_i\) is calculated as follows:
\[
B_i = \text{Top-}N_{H(\cdot|\vx)}\left\{\left(t, H\left(\vy_t \mid \vx, \vy_{<t} \right)\right) \mid t \in \mathcal{T}_i\right\}
\]
where \(H(\vy_t) = -\log \pi_{\theta}(\vy_t \mid \vx,\vy_{<t})\) denotes the entropy of token \(\vy_t\). 
Additionally, we mask tokens near the end of sequences, as the model is expected to explore different reasoning paths rather than simply revisiting previous answers.

\vpara{Expansion.} Given the selected tokens for each tree, we continue the generation process from these tokens to form new branches. For each forking point \( t \), with the prefix \( \vy_t \) and prompt \( \vx \), we generate \( T \) different candidate responses until completion:
\[
Y_{\text{new}}^{(i)} \sim \{\pi_\theta\left(\cdot \mid \vx, \vy_{<t}\right), \ \text{for} \ (t,\cdot) \in B_i\}^{T}
\]
where \( \vy_{<t} \) denotes the prefix response before token \( t \). This results in \( M \times N \times T \) tree nodes in total( \(N \times T\) for each tree $\mathcal{T}_i$ ), each corresponding to a new response. The tree structure \( \mathcal{T}_i \) is updated to include these new nodes:
\[
\mathcal{T}_i \leftarrow \mathcal{T}_i \cup Y_{\text{new}}^{(i)}
\]
After initialization, the forking and expansion process is repeated for \( L \) iterations, leading to \( M \times (T \times N \times L+1 ) \) leaves (responses). We denote this entropy-guided tree search as an \( (M, N, L, T) \)-tree, where \( M \) is the number of parallel trees, \( N \) is the number of forked points per iteration, \( L \) is the number of iterations, and \( T \) is the branching factor at each forking point.

In comparison to independent multi-chain sampling, the \treemodel is capable of producing a larger number of diverse responses under the same inference cost, as illustrated in \Cref{fig:tree_vs_mcts_iid_resp_num}. This offers the potential to enhance RL performance by learning from more varied responses.

\begin{figure}[t]
        \centering
        \includegraphics[width=0.45\textwidth]{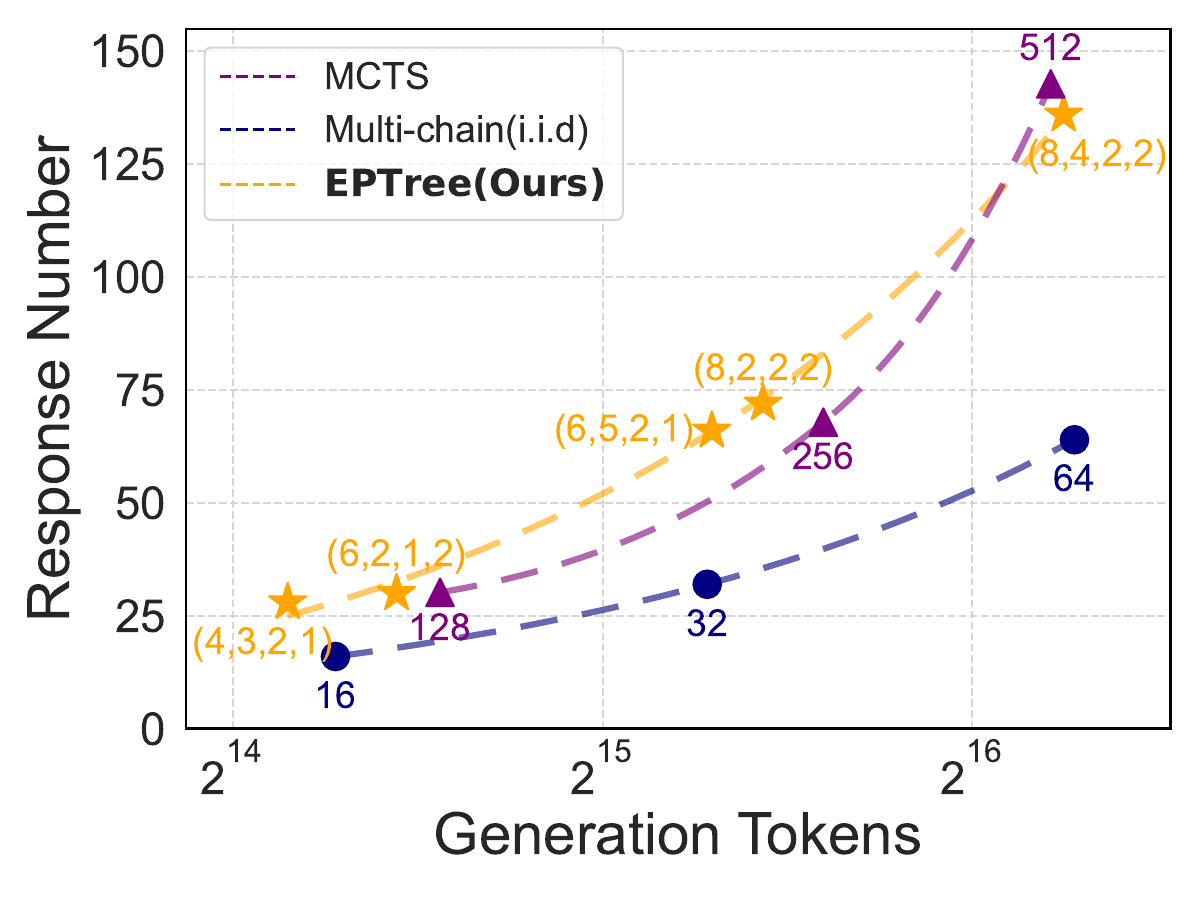}
        \label{fig:tree_vs_iid_leaf_num}
    \caption{Generation diversity comparison of \treemodel, MCTS, and i.i.d. multi-chain sampling. Both \treemodel and MCTS produce approximately $2\times$ different responses compared to i.i.d. multi-chain sampling. }
    \label{fig:tree_vs_mcts_iid_resp_num}
\end{figure}

\begin{algorithm*}[t]
    \caption{\model}
    \label{alg:entropy-guided-tree-search}
    \begin{algorithmic}{}
    \Statex \textbf{Input}: Prompt $\vx$, Policy $\pi_\theta$, Number of Trees $M$, Forking Points $N$, Iterations $L$, Branches $T$
    \Statex \textbf{Output}: Optimized Policy $\pi_{\theta}^*$
        \For{$i = 1$ to $M$} \Comment{\textcolor{blue}{{Create $M$ trees with $M(1+NLT)$ leaves}}}
            \State $Y^{(i)} \leftarrow \{\vy_i \sim \pi_\theta(\cdot \mid \vx)\}$, $\mathcal{T}_i \leftarrow \{Y^{(i)}\}$ \Comment{\textcolor{blue}{{Initialization}}}
            \For{$l = 1$ to $L$} \Comment{\textcolor{blue}{Iterative Expansion}}
                \State $H(\vy_t) \leftarrow -\log \pi_{\theta}(\vy_t \mid \vx, \vy_{<t}), \forall\ t \in \mathcal{T}_i$
                \State $B_{i, l} \leftarrow \text{Top-}N_{H(\cdot|\vx)}\left\{\left(t, H\left(\vy_t \mid \vx, \vy_{<t} \right)\right) \mid t \in \mathcal{T}_i\right\}$
                \For{each selected forking point $(t, \cdot) \in B_{i, l}$}
                    \State $Y_{\text{new}}^{(i,l)} \sim \{\pi_\theta\left(\cdot \mid \vx, \vy_{<t}\right), \ \text{for} \ (t,\cdot) \in B_{i,l}\}^{T}$
                    \State $\mathcal{T}_i \leftarrow \mathcal{T}_i \cup Y_{\text{new}}^{(i,l)}, j\in \{1,\cdots, T\}$
                \EndFor
            \EndFor
        \EndFor
        
        \For{each step $s_n$ in $\mathcal{T}_i$}  \Comment{\textcolor{blue}{Process Reward Calculation}}
            \State $V(s_n) \leftarrow \dfrac{1}{|L(s_n)|} \sum_{l \in L(s_n)} \mathbf{1}(l \text{ is correct})$
            \State $R(s_n) \leftarrow \underbrace{|L(s_n)|^{-1/2}}_{\text{Re-weight Factor}}\cdot [\underbrace{V(s_n) - V(\text{root})}_{\text{\textbf{G}lobal \textbf{A}dvantage}} + \underbrace{V(s_n) - V(p(s_n)}_{\text{\textbf{L}ocal \textbf{A}dvantage}}]$
        \EndFor
    \State Update $\pi_{\theta}$ using RL with process reward by Policy Gradient
    \end{algorithmic}
\end{algorithm*}

\subsection{Reinforcement Learning with \treemodel}
In this part, we show how to integrate \treemodel into RL training.
Beyond the superior potential to find correct trajectories, hierarchical tree structures also provide more fine-grained process supervision for intermediate steps for RL training.

\subsubsection{Process Supervision from Tree Search}
At each step, we estimate the value using Monte Carlo methods.
Specifically, for a step \( s_n \) (corresponding to a node in the tree), let \( L(s_n) \) denote the set of all leaf nodes that are descendants of node \( s_n \) (including node \( s_n \) itself if it is a leaf). The value \( V(s_n) \) of node \( s_n \) is computed as the ratio of correct leaf nodes among its descendants:
\[
    V(s_n) = \frac{1}{|L(s_n)|} \sum_{l \in L(s_n)} \mathbf{1}(l \text{ is correct})
\]
This value reflects the potential of the node \( s_n \) to lead to correct answers. Based on value, process supervision for each step is defined using advantages, i.e., how much a step is better than other steps, which include both \emph{global} and \emph{local advantages}.

The \emph{global advantage} of a step \( s_n \) represents its potential to lead to a correct outcome compared to the overall correctness ratio of all samples. Without loss of generality, assume that there exists a virtual root node for all subtrees, and the value of the root node $V(root)$ represents the average correctness of all generated responses. The global advantage is then computed as:
\begin{equation}
G_A(s_n) = V(s_n) - V(\text{root})
\label{eq:global_adv}
\end{equation}
Essentially, the global advantage of \(s_n\) is equivalent to the normalized value, which can be obtained by first normalizing the reward across all leaf nodes---subtracting the average reward---and then calculating the average rewards for \(L(s_n)\).

The \emph{local advantage} of a node \( s_n \) quantifies the improvement the step $s_n$ provides compared to its parent step $p(s_n)$
and is defined as:
\begin{equation}
L_A(s_n) = V(s_n) - V(p(s_n))
\label{eq:local_adv}
\end{equation}
where \( V(p(s_n)) \) denotes the value of the node \( p(s_n) \). \( L_A(s_n) > 0 \) indicates that step \( s_n \) is more likely to lead to the correct result than its parent node, suggesting that this step should be encouraged, and vice versa.

To compute the final reward for each step, we combine the global and local advantages as follows:
\begin{equation}
R(s_n) = G_A(s_n) + L_A(s_n)
\label{eq:process_reward}
\end{equation}


The definition could be viewed as a special case of Generalized Advantage Estimation (GAE)~\cite{schulman2015high}. The general form of GAE in LLM is defined as:
\begin{equation}
    A(s_n \rightarrow s_{n+t}) = \gamma^t V(s_{n+t}) - V(s_n) 
\end{equation}
where $t$ represents any integer time step, and $\gamma$ is typically set to 1 in most cases. Thus, the local advantage defined in Eq~\ref{eq:local_adv} corresponds to the case where $t=1$, while the global advantage corresponds to the case where $n=0$. 
Inspired by the GAE, the process supervision signal can be defined as a more generalized format by considering not only the root and direct parent but also all of its ancestor nodes in the trajectory:
\[
R_{GAE}(s_n) = \sum_{j \in P(s_n)} \lambda_j \cdot [V(s_n) - V(s_j)]
\]
where \( P(s_n) \) represents the set of ancestor nodes of \( s_n \) and $\lambda_j$ denotes the weight of each step. In this work, we focus on a special format that only considers the direct parent and the root node. 

\subsubsection{Training with Process Superivison}
At each iteration, we first utilize \treemodel described in \Cref{sec:entropy-tree} to generate $M$ trees \( \mathcal{T} = \{\mathcal{T}_i\}_i^{M} \) for the prompt, where each leaf node in \( \mathcal{T} \) together with its all prefix corresponds to a complete sequence. 
For each step in the tree, the process supervision signal is assigned based on \Cref{eq:process_reward} to reflect its importance.

The sequences \(\{ S_1, S_2, \ldots, S_{(M\times N \times L \times T + M)} \} \) extracted from the trees are used for RL training. As all non-leaf steps appear in multiple sequences and will be repeatedly computed in optimization, we downweight the reward of these steps to prevent overfitting. The reward for each non-leaf step is modified by dividing the square root of the number of leaf nodes in its subtree: \( R(s_n) = R(s_n) / \sqrt{|L(s_n)|} \). This adjustment leads to improved performance in our experiments. The overall pipeline of \model is illustrated in Alg~\ref{alg:entropy-guided-tree-search}.

\section{Experiment}

\begin{figure}[t]
    \centering
    \includegraphics[width=0.45\textwidth]{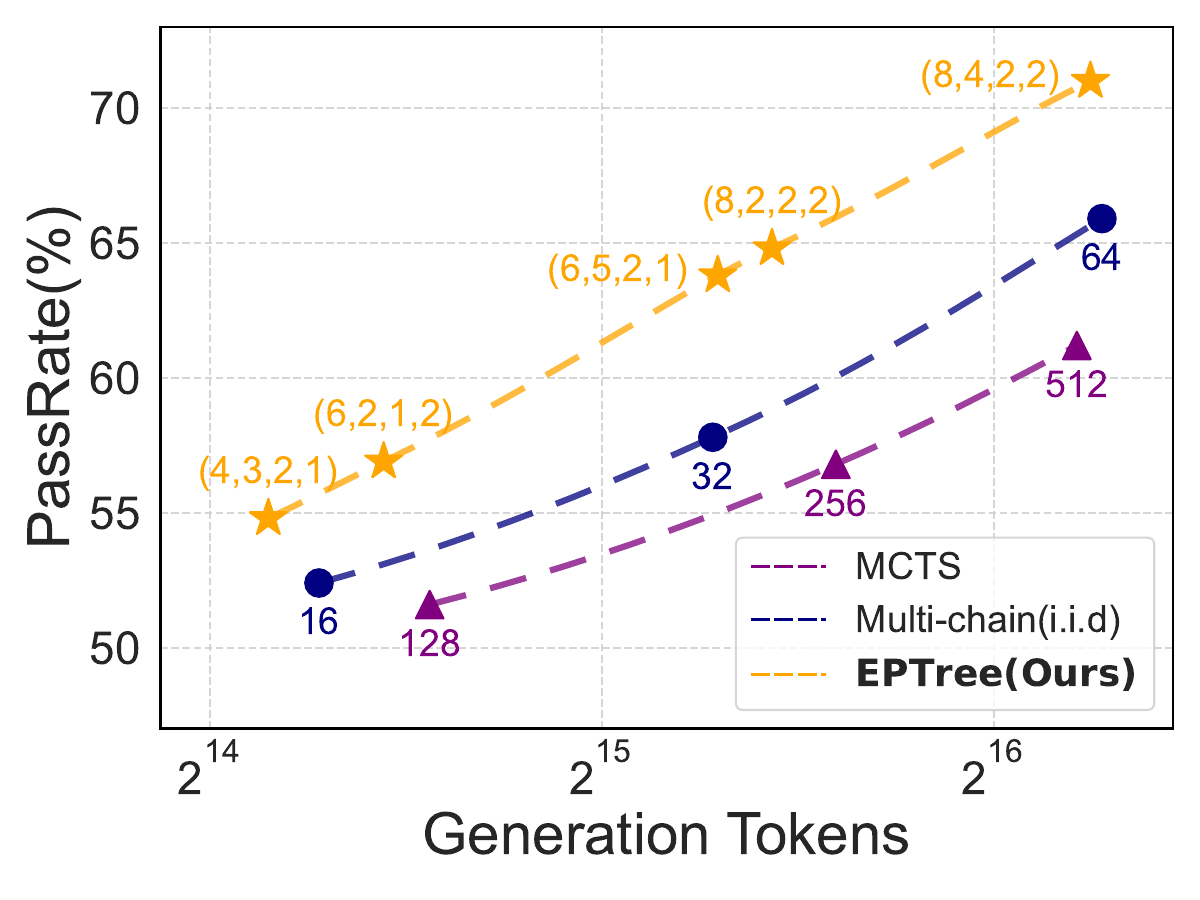}
    \label{fig:tree_vs_iid_performance}
    \caption{Search performance of \treemodel, MCTS, and multi-chain sampling on Omni-MATH-500 using Qwen-2.5-14B-SFT. \treemodel consistently outperforms all baselines under the different inference costs. 
    }
    \vspace{-4mm}
    \label{fig:tree_vs_mcts_iid}
\end{figure}

\begin{table*}[!t]
    \caption{Experiment results on math reasoning tasks. We report Accuracy(\%) for all datasets.}
    \centering
    \resizebox{0.95\textwidth}{!}{%
        \begin{tabular}{lccccccc}
        \toprule[1.2pt]
             & MATH500 & \makecell{Omni-MA\\TH-500} & AIME2024 & AMC & \makecell{Olympiad\\Bench} & \makecell{LiveCode\\Bench} & Avg \\
         \midrule
         GPT-4o & 76.6 & 26.8 & 9.3 & 45.8 & 43.3 & 29.5 & 38.6 \\
         Llama-3.1-8B-Instruct & 52.8 & 15.0 & 10.9 & 22.6 & 15.6 & 11.6 & 21.4 \\
         Llama-3.3-70B-Instruct & 73.9 & 27.9 & 24.2 & 50.9 & 35.7 & 25.5 & 39.7 \\
         GLM4-9B-chat & 50.1 & 12.9 & 1.7 & 17.2 & 14.7 & 16.5 & 18.9 \\
         Qwen-2.5-7B-Instruct & 76.5 & 26.0 & 13.3 & 41.9 & 35.0 & 16.8 & 34.9 \\
         Qwen-2.5-Math-7B-Instruct & 82.7 & 29.7 & 16.7 & 50.6 & 40.7 & 8.1 & 38.1 \\
         Qwen-2.5-14B-Instruct & 78.9 & 28.7 & 13.7 & 54.5 & 41.8 & 27.7 & 40.9 \\
         \midrule
         SFT (GLM-9B) & 56.0 & 18.2 & 8.3 & 29.2 & 22.5 & 14.2 & 24.7  \\
         ChainRL (GLM-9B)  & 63.0 & 21.8 & 6.1 & 31.6 & 23.9 & 16.6 & 27.2 \\
         \model (GLM-9B) & 64.5 & 20.8 & 11.4 & 38.5 & 24.8 & 15.8 & \textbf{29.3} \\
         \midrule
         SFT (Qwen-2.5-14B) & 76.6 & 29.5 & 10.6 & 48.0 & 36.9 & 14.5 & 36.0 \\
         ChainRL (Qwen-2.5-14B) & 81.6 & 32.7 & 22.2 & 53.9 & 41.1 & 18.2 & 41.6\\
         \model (Qwen-2.5-14B) & 81.7 & 36.7 & 28.0 & 55.9 & 44.6 & 20.8 & \textbf{44.5} \\
         \midrule
         SFT (R1-Distilled-Qwen-2.5-7B) & 94.0 & 47.8 & 55.9 & 85.5 & 54.4 & 43.9 & 63.6 \\
         ChainRL (R1-Distilled-Qwen-2.5-7B) & 93.6 & 48.1 & 59.7 & 85.5 & 54.5 & 46.1 & 64.5\\
         \model (R1-Distilled-Qwen-2.5-7B) & 94.4 & 49.8 & 60.8 & 85.0 & 57.1 & 47.4 & \textbf{65.8} \\
         \bottomrule[1.2pt]
        \end{tabular}
    }
    \label{tab:experiments}
\end{table*}

\begin{figure*}[!h]
    \centering
    \subfloat[Experiments on  Qwen-2.5-14B \label{subfig:a}]{
        \begin{minipage}{\textwidth}
            \centering
            \begin{minipage}{0.32\textwidth}
                \centering
                \includegraphics[width=\textwidth]{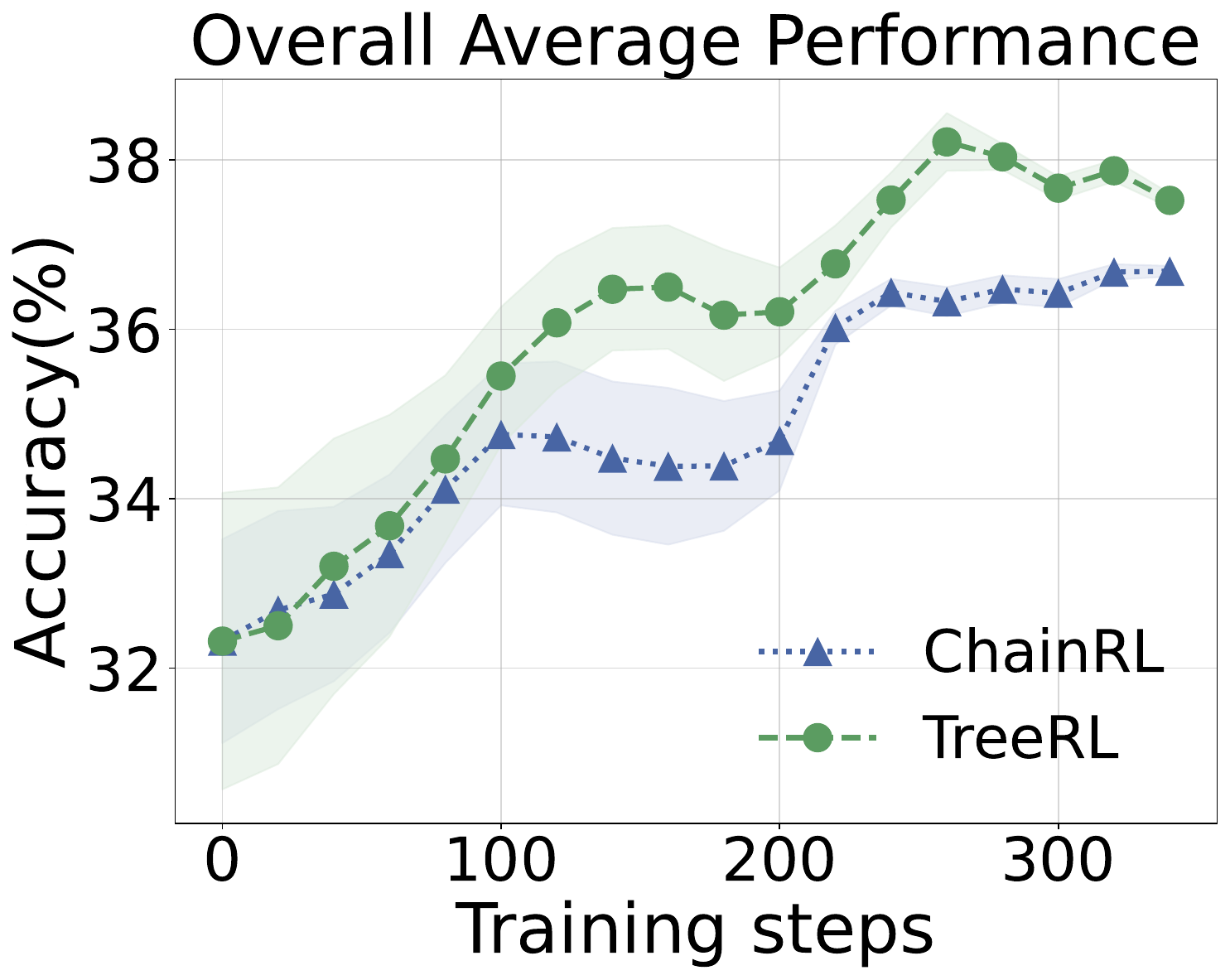}
                \vspace{-5mm}
                \label{fig:first_image}
            \end{minipage}
            \hfill
            \begin{minipage}{0.32\textwidth}
                \centering
                \includegraphics[width=\textwidth]{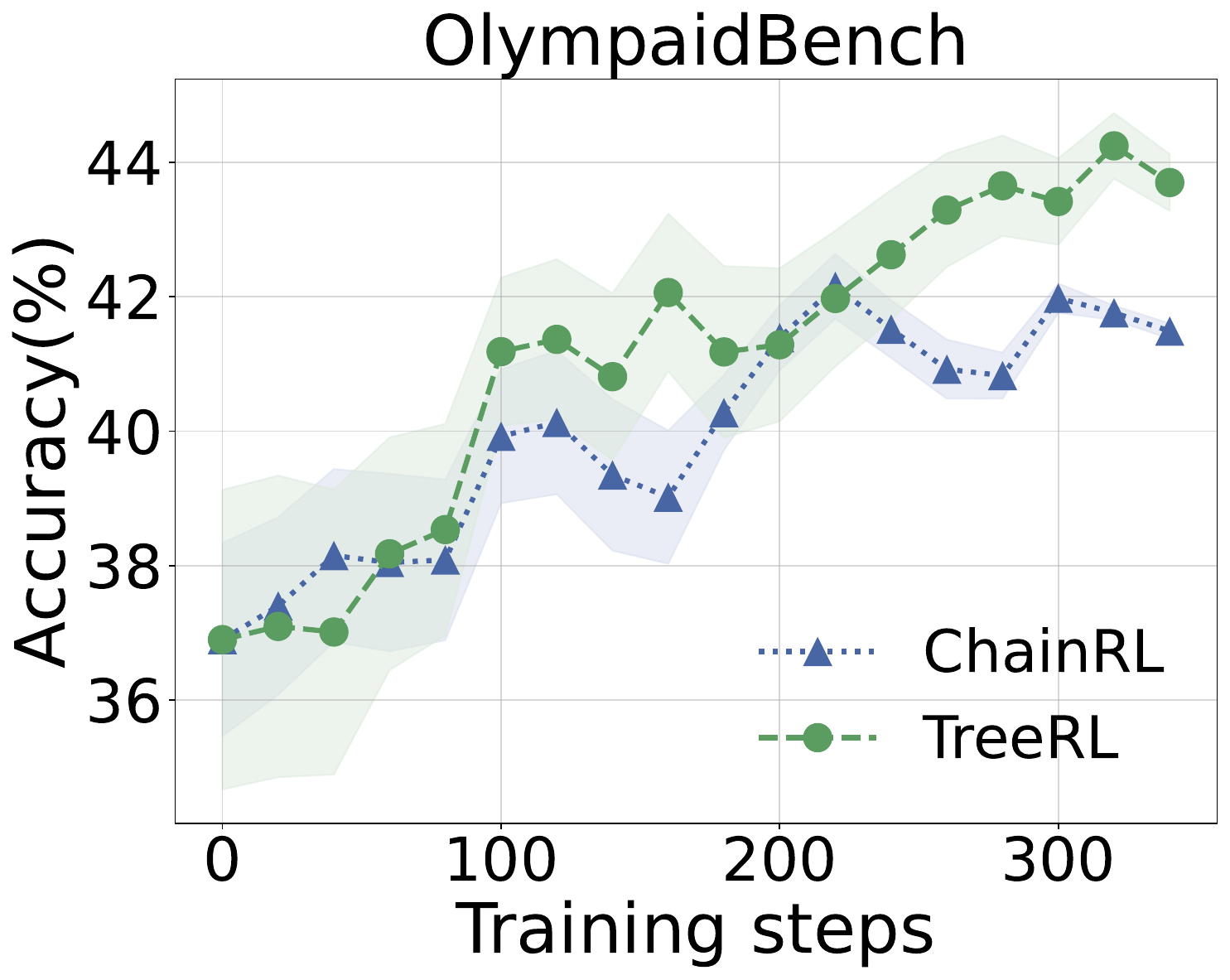}
                \vspace{-5mm}
                \label{fig:second_image}
            \end{minipage}
            \hfill
            \begin{minipage}{0.32\textwidth}
                \centering
                \includegraphics[width=\textwidth]{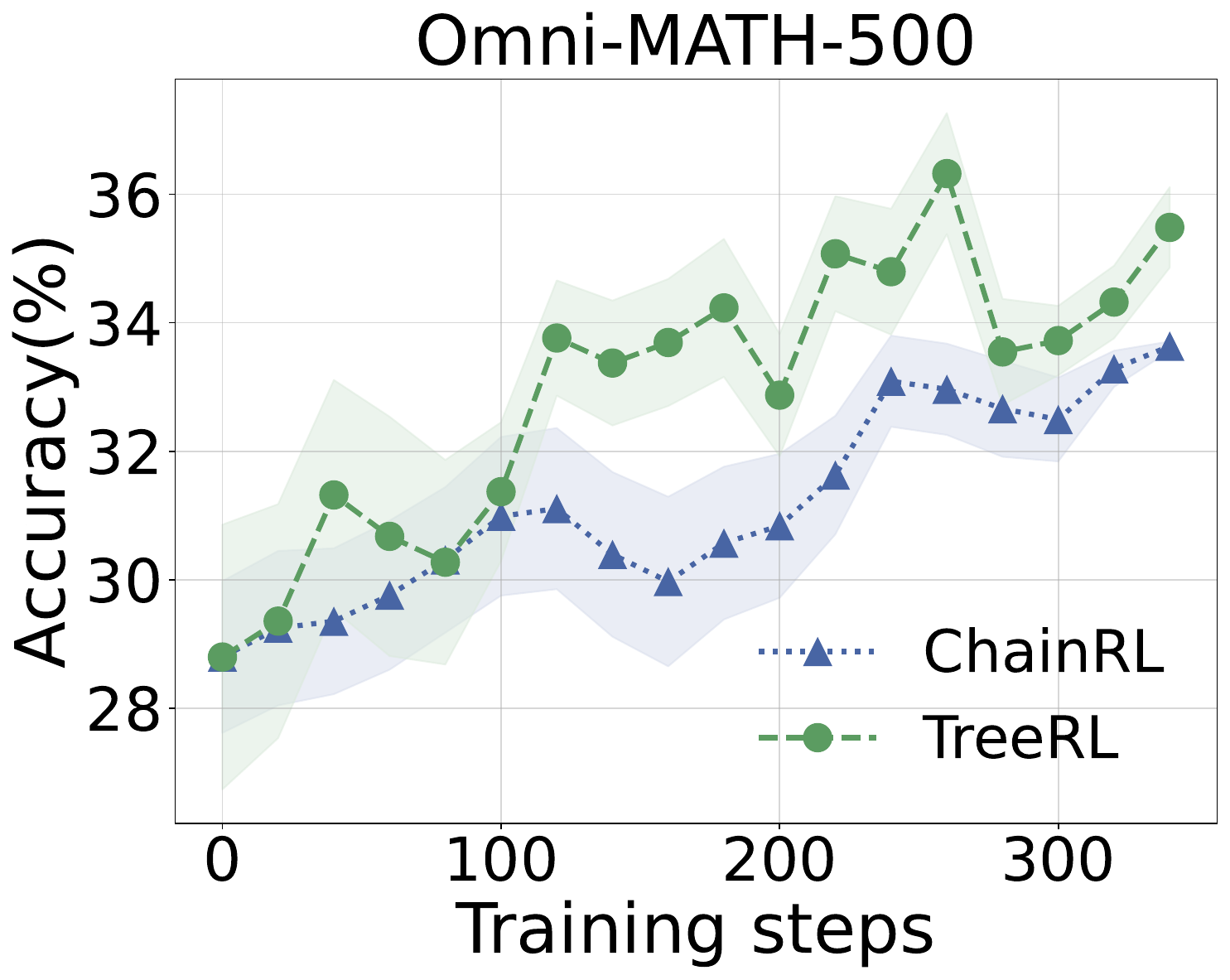}
                \vspace{-5mm}
                \label{fig:third_image}
            \end{minipage}
        \end{minipage}
    }
    \vspace{1mm}
    \subfloat[Experiments on GLM4-9B \label{subfig:b}]{
        \begin{minipage}{\textwidth}
            \centering
            \begin{minipage}{0.32\textwidth}
                \centering
                \includegraphics[width=\textwidth]{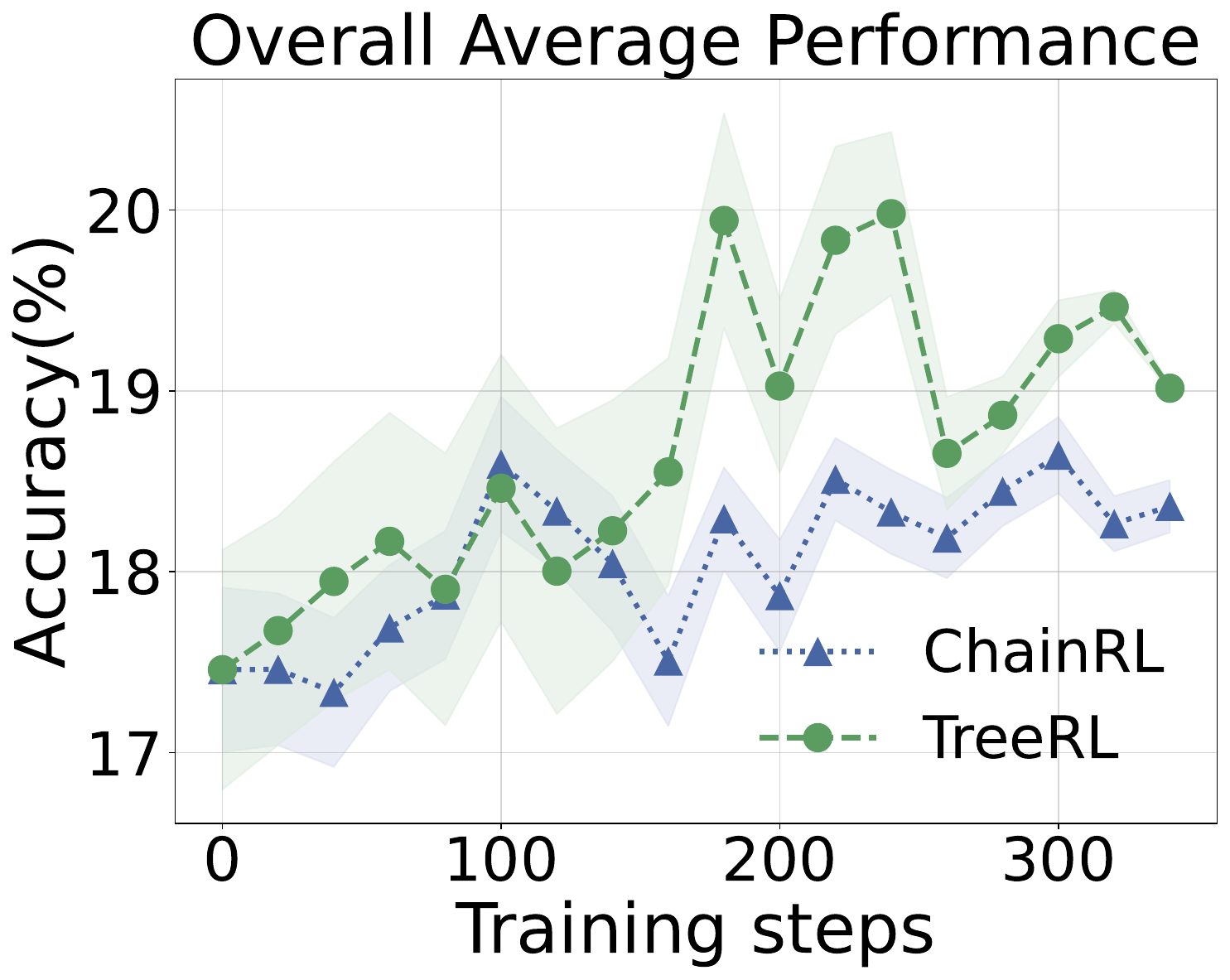}
                \vspace{-5mm}
                \label{fig:fourth_image}
            \end{minipage}
            \hfill
            \begin{minipage}{0.32\textwidth}
                \centering
                \includegraphics[width=\textwidth]{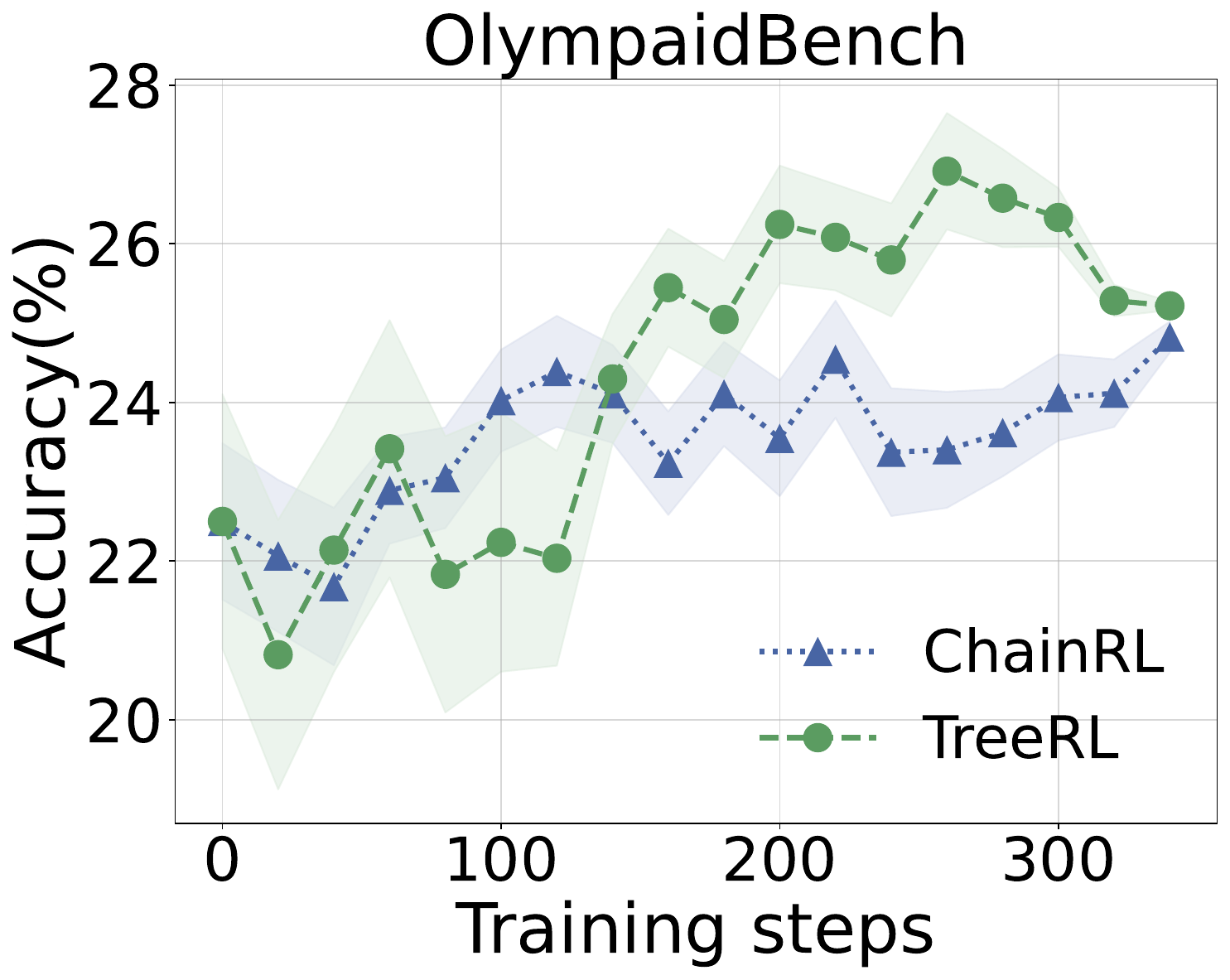}
                \vspace{-5mm}
                \label{fig:fifth_image}
            \end{minipage}
            \hfill
            \begin{minipage}{0.32\textwidth}
                \centering
                \includegraphics[width=\textwidth]{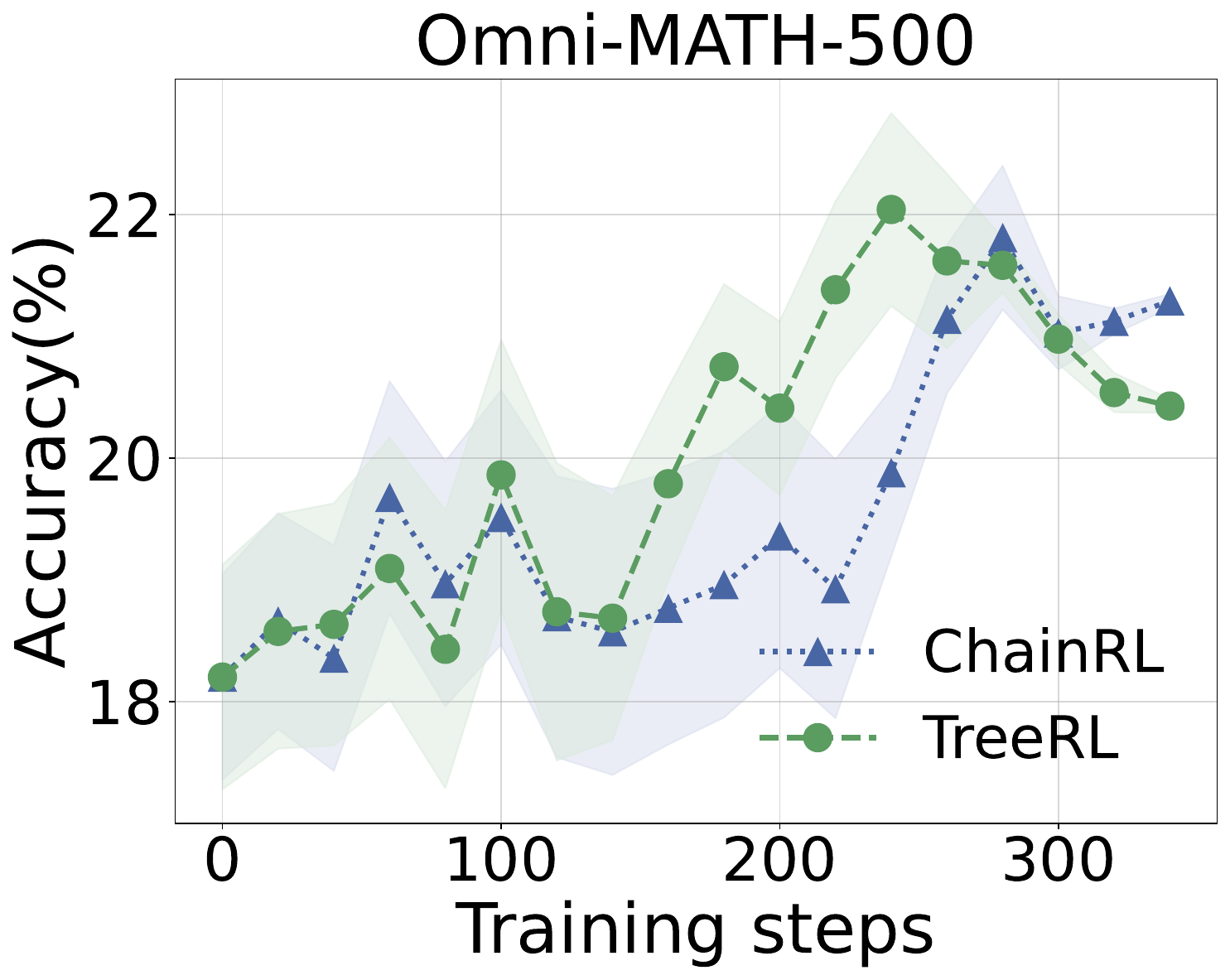}
                \vspace{-5mm}
                \label{fig:sixth_image}
            \end{minipage}
        \end{minipage}
    }
    \caption{Performance comparison between \model and ChainRL during training. We report the average performance across all six datasets (Left), OlympiadBench (Middle), and Omni-MATH-500 (Right). The results for the other benchmarks can be found in the \Cref{sec:appendix:other_benchmark}.}
    \label{fig: RL-results}
\end{figure*}

\subsection{Setup}
\label{sec:exp_setup}
\para{Training Details.} 
We train SFT models based on Qwen-2.5-14B~\citep{qwen2_5} and GLM4-9B~\citep{glm2024chatglm} as initialization for RL training and finetune them using the public dataset from \citep{hou2025advancing} for 2 epochs. 
For DeepSeek-R1-distilled-Qwen-2.5-7B, we directly use the open-source model from~\cite{guo2025deepseek} without further finetuning.
Then, to identify the best tree search setting for RL under a given inference cost, we investigate various combinations of hyperparameters $(M, N, L, T)$ using the trained Qwen-2.5-14B-SFT model on the Omni-MATH-500 dataset with PassRate as the target metric, and the results can be found in Table~\ref{tab:tree-parameter-evaluation} in the Appendix. 
The baseline i.i.d multi-chain samplings correspond to setting $N = L = T = 0$.

For the RL training, we compare the performance of RL using multi-chain sampling, referred to as ChainRL, to the proposed \model under the same inference cost derived from the previous search to ensure a fair comparison. 
For \model, we used sampling parameters $(M, N, L, T) = (6, 2, 1, 2)$, generating 30 responses per prompt. This is comparable to multi-chain sampling with 16 responses per prompt, given similar generation token budgets.  Each iteration used 16 prompts for a single gradient update, resulting in a batch size of 256 for ChainRL and 480 for \model. Therefore, \model actually performs reinforcement learning with the same inference cost but with more training computation during training. 
We use rule-based reward based on the correctness of the final answer, i.e., $+1$ for correct and $0$ for wrong.
The KL efficiency is set to $\beta = 10^{-4}$, with a learning rate of $1.5 \times 10^{-6}$. During sampling, the temperature is 1.2, the top-$p$ value is 0.95, and the maximum sequence length is 8,192.
For RL training, the training data all come from publicly available datasets, including MATH-train~\citep{hendrycksmath2021} and NuminaMath~\citep{li2024numinamath}, and we sample a subset for training.

\para{Evaluation}
We use PassRate to evaluate the effectiveness of the proposed tree search algorithm. PassRate measures the potential of a model to reach at least one correct answer among all the generated solutions. 
It should be noted that we compare tree sampling and chain sampling under similar inference budgets, i.e., generation tokens.
Specifically, given an evaluation dataset $\mathcal{D}$, for each sample $d \in \mathcal{D}$, let $\mathcal{L}_d$ be the set of generated responses and $c(l)\in \{0,1\}$ be a binary indicator of correctness for a response $l$. The PassRate is computed as:
\[
\text{PassRate} = \frac{1}{|\mathcal{D}|} \sum_{d \in \mathcal{D}} \max_{l \in \mathcal{L}_d} c(l)
\]
For reinforcement learning, we evaluate the policy model on 6 challenging reasoning benchmarks using greedy sampling, 
including MATH\citep{hendrycksmath2021}, Omni-MATH \citep{gao2024omni}, AIME2024\citep{li2024numinamath}, AMC\citep{li2024numinamath}, OlympiadBench\citep{he2024olympiadbench}, and LiveCodeBench~\citep{jain2024livecodebench}. For the MATH dataset, we use a subset known as MATH500 based on the split defined by \citep{lightmanlet}. For Omni-MATH, we randomly sample a subset for evaluation, named Omni-MATH-500, which consists of 500 examples for efficient yet comprehensive assessment. Each dataset undergoes multiple evaluations to minimize variance; for instance, we evaluate MATH500 and Omni-MATH-500 three times each, while AIME is evaluated 32 times, considering AIME2024 consists of only 30 questions. For LiveCodeBench, we report the performance using the data between 202407 to 202411.


\subsection{Results of \treemodel}
\Cref{tab:tree-parameter-evaluation} shows the overall performance on Omni-MATH-500 and illustrates the efficiency and effectiveness trade-off across different sampling methods. 
\treemodel demonstrates a significant advantage over the baseline multi-chain sampling method and the MCTS method under the same inference cost. And \treemodel shows consistently better performance across different generation costs and outperforms the multi-chain sampling by around 3\% in PassRate. Compared to MCTS, \treemodel demonstrates a more significant advantage, with the margin widening as the inference cost increases.

\subsection{Results on RL training}
Table~\ref{tab:experiments} presents the performance of various sampling strategies in RL training. Notably, \model equipped with \treemodel sampling outperforms traditional multi-chain sampling across different benchmarks. In particular, Figure~\ref{fig: RL-results} illustrates the evaluation performance over various training steps. While both sampling strategies exhibit similar performance during the early stages of training, \model begins to show an advantage around 100 steps and continues to improve consistently. This indicates that the \model achieves better prompt efficiency by delivering enhanced performance with the same number of training prompts. Overall, these results underscore the promise of integrating tree search and process supervision into RL training. 
It is observed that both RL methods show only minor improvements over the SFT baseline. This could be due to the fact that we selected the RL training data based on Qwen-2.5-14B, which is likely much easier for the R1-Distilled-Qwen-2.5-7B model, thereby limiting the potential performance gains in RL.

\subsection{Ablation Study on \treemodel Sampling}
\vpara{Effects of entropy-based forking.} Table~\ref{tab:entropy-fork-token-ablation} illustrates the performance across different strategies when forking new branches in tree expansion. 
\treemodel shows better PassRate than random forking with fewer generation tokens, which demonstrates the advantage of \treemodel. In addition, both tree sampling strategies outperform multi-chain sampling, offering the potential of tree search.

\begin{table}[t]
    \centering
    \caption{Ablation of \treemodel on Omni-MATH-500. We compare forking branches using random and entropy-based strategies on Qwen-2.5-14B-SFT. $(16,0,0,0)$ corresponds to multi-chain sampling. \textit{Entropy} denotes whether to use entropy-based forking strategy. \treemodel model shows better PassRate performance yet with fewer generation tokens.}
    \resizebox{0.5\textwidth}{!}{
        \begin{tabular}{cccc}
        \toprule[1.2pt]
        $(M,N,L,T)$ & Entropy & PassRate $\uparrow$ & \#Token $\downarrow$ \\
        \midrule
        $(6,2,1,2)$ & $\checkmark$ & \textbf{56.9} & \textbf{22268} \\
        $(6,2,1,2)$ & \ding{55} & 54.8 & 24213 \\
        $(16,0,0,0)$ & - & 52.4 & 19858 \\
        \midrule
        $(8,4,2,2)$ & $\checkmark$ & \textbf{71.0} & \textbf{77768} \\
        $(8,4,2,2)$ & \ding{55} & 70.0 & 89452 \\
        $(64,0,0,0)$ & - & 67.4 & 79367 \\
        \bottomrule[1.2pt]
        \end{tabular}
    }
    \label{tab:entropy-fork-token-ablation}
\end{table}

\vpara{Case study on forking tokens.} 
To help better understand \treemodel, we conduct case studies to analyze what type of tokens tend to be selected.
We examine frequency by identifying the top-10 tokens that most often serve as forking points, and the results are illustrated in \Cref{fig:top_10_tokens}. It can be observed that mathematical operators ($\backslash($), logical conjunctions, and transitional terms (``Since'', ``But'') are frequently selected. Notably, the token ``wait'' appears frequently, as o1-like models often use it for self-reflection steps. 

\begin{figure}[t]
    \centering
    \includegraphics[width=0.95\linewidth]{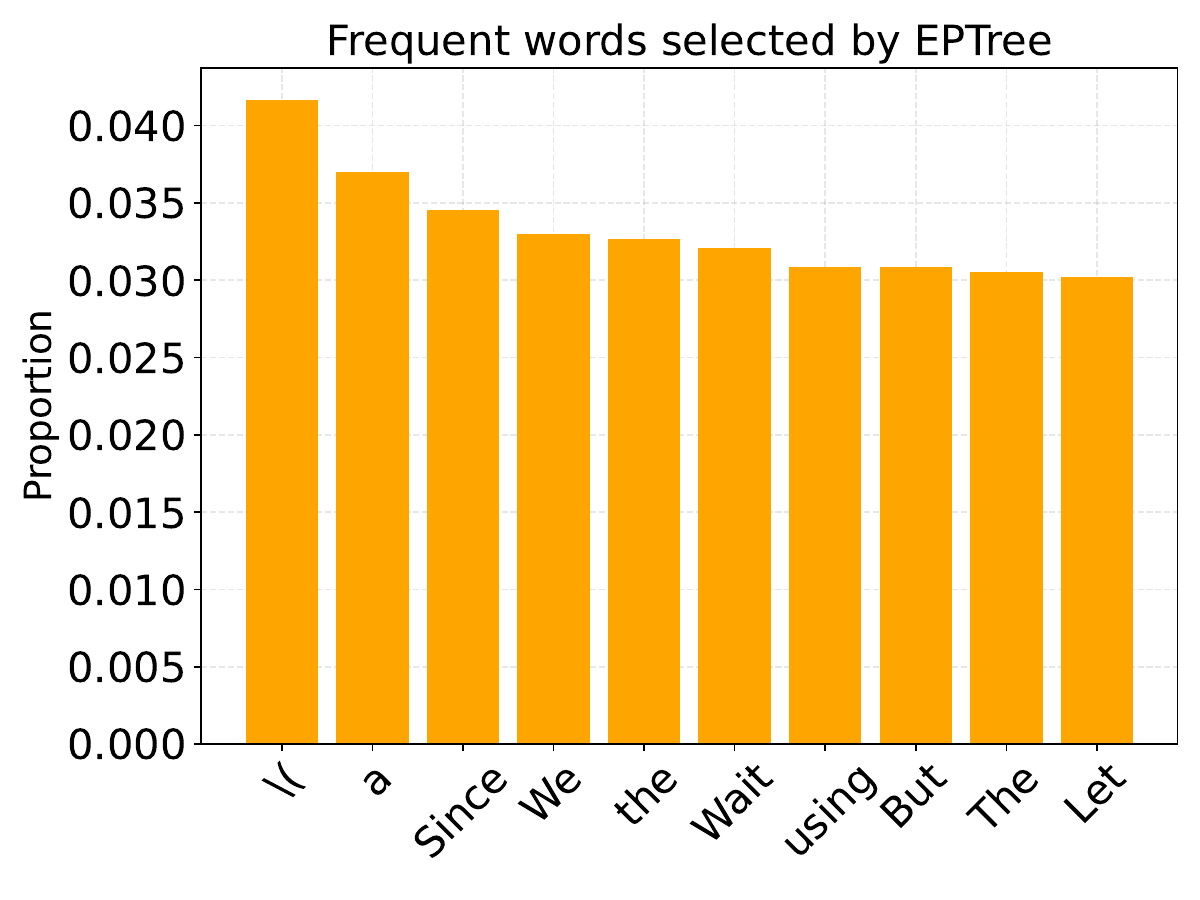}
    \caption{Frequent words of sampled forking tokens in \treemodel sampling on Omni-MATH-500.}
    \label{fig:top_10_tokens}
\end{figure}

\hide{
\begin{table*}[ht]
    \caption{Ablation on the effectiveness of different process signals for RL training using Qwen2.5-14B. $G_A$ and $L_A$ denote the global and local advantages, respectively. $n$ refers to the number of leaf nodes that can be reached from the given step. \#Response denotes the number of responses used for training.}
    \centering
    \resizebox{\textwidth}{!}{%
    \begin{tabular}{cc|cccccc|c}
    \toprule[1.2pt]
       Reward  & \#Responses & MATH500 & \makecell{Omni-MA\\TH-500} & AIME2024 & AMC & \makecell{Olympiad\\Bench} & LiveCodeBench & Avg \\
    \midrule
    $(G_A + L_A)/\sqrt{n}$ & 30  & \textbf{81.7} & \textbf{36.7} & \textbf{28.0} & \textbf{55.9} &  \textbf{44.6} & \textbf{20.8} &\textbf{44.5} \\
    $G_A+L_A$ & 30 & 81.5 & 32.0 & 24.1 & 56.2 & 42.1 & & \\
    $G_A$ & 30 &  80.1 & 35.1 & 24.7 & 55.5 & 42.8 & 20.7 & 43.2 \\
    $(G_A + L_A)/\sqrt{n}$ & 16 & 80.1 & 32.5 & 24.5 & 52.9 & 41.7 & 15.8 & 41.3 \\
     \bottomrule[1.2pt]
    \end{tabular}%
    }
    \label{tab:ablation_rl}
\end{table*}
}

\begin{table*}[ht]
    \caption{Ablation on the effectiveness of different process signals for RL training using Qwen2.5-14B. $G_A$ and $L_A$ denote the global and local advantages, respectively. $n$ refers to the number of leaf nodes that can be reached from the given step. \textit{\#Response} denotes the number of responses used for training.}
    \centering
    \resizebox{\textwidth}{!}{%
    \begin{tabular}{cc|cccccc|c}
    \toprule[1.2pt]
       Reward  & \#Responses & MATH500 & \makecell{Omni-MA\\TH-500} & AIME2024 & AMC & \makecell{Olympiad\\Bench} & \makecell{LiveCode\\Bench} & Avg Gain \\
    \midrule
    $(G_A + L_A)/\sqrt{n}$ & 30  & \textbf{81.7} & \textbf{36.7} & \textbf{28.0} & 55.9 & \textbf{44.6} & \textbf{20.8} & - \\
    $G_A+L_A$ & 30 & 81.5 & 32.0 & 24.1 & \textbf{56.2} & 42.1 & 19.7 & -1.9 $\downarrow$ \\
    $G_A / \sqrt{n}$ & 30 &  80.1 & 35.1 & 24.7 & 55.5 & 42.8 & 20.7 & -1.3 $\downarrow$ \\
    $(G_A + L_A)/\sqrt{n}$ & 16 & 80.1 & 32.5 & 24.5 & 52.9 & 41.7 & 15.8 & -3.2 $\downarrow$ \\
     \bottomrule[1.2pt]
    \end{tabular}%
    }
    \label{tab:ablation_rl}
\end{table*}


\subsection{Ablation on \model}
To test the effectiveness of process supervision in \model, we compare different designs of process supervision signals. Additionally, since the \model uses more traces than RL with multi-chain sampling, we also evaluate how the increased training cost impacts performance. The results are presented in Table~\ref{tab:ablation_rl}.
The results demonstrate that RL with reweighted local and global advantage achieves the best performance. Removing either component could lead to a decline in performance. Moreover, using a subset of sampled responses shows less improvement compared to utilizing the entire set. This suggests that, in addition to higher PassRate and process supervision, tree search enhances RL performance by enabling more training within the same inference cost.


\section{Related Work}

\subsection{RL for LLM Reasoning}
Recent advanced~\citep{guo2025deepseek,zhu2024deepseek,openaio1,ouyang2022training} have demonstrated the effectiveness of reinforcement learning in LLM alignment and reasoning. Most existing methods train a reward model or use rule-based rewards for the entire response. In parallel, process supervision\citep{lightman2023letsverifystepstep} has shown a particularly promising performance than outcome supervision. 
Most existing works resort to training a process reward model (PRM) based on human- or auto-annotated process signals\citep{lightman2023letsverifystepstep, wang2024mathshepherdverifyreinforcellms,luo2024improvemathematicalreasoninglanguage,setlur2024rewardingprogressscalingautomated} and apply the static PRM for RL training~\citep{shao2024deepseekmathpushinglimitsmathematical,wang2024mathshepherdverifyreinforcellms}. However, a static PRM could suffer from distribution shift and reward hacking as RL training progresses. \citep{kazemnejad2024vineppounlockingrlpotential} overcomes this problem by conducting Monte Carlo rollouts to estimate the value for each step, but it suffers from high computation cost with around quadratic computational complexity and thus hinders scalability.

\subsection{Tree Search for LLM Reasoning}
Tree search has mainly been explored in LLM alignment and inference, especially for data synthesis and offline preference training.
\citet{xie2024monte} employs MCTS to generate step-level preference pairs for DPO~\citep{rafailov2024direc}. ~\citet{chen2024alphamathzeroprocesssupervision,feng2024alphazeroliketreesearchguidelarge,zhang2024restmctsllmselftrainingprocess} employ MCTS to iteratively generate high-quality SFT data or produce process supervision signals for training. Other works explore reward-guided search~\citep{snell2024scalingllmtesttimecompute,yao2023treethoughtsdeliberateproblem,long2023largelanguagemodelguided} to boost the performance in the inference stage.  
However, few efforts have been devoted to studying how to explicitly integrate tree search into reinforcement learning training like AlphaZero~\citep{silver2017mastering} to improve LLM reasoning.

\section{Conclusion}
This work presents \model, an RL approach that combines tree search with process supervision to enhance LLM reasoning. \treemodel improves response diversity and performance over traditional methods like MCTS and i.i.d multi-chain sampling. Then, we conduct reinforcement learning with \treemodel and the derived process supervision from tree search. 
Experiments on math reasoning tasks show that the \model outperforms existing techniques, highlighting the potential of RL with tree search to advance LLM in complex reasoning tasks.

\section{Limitation}
In this work, we propose to improve reinforcement learning with on-policy tree search. While this approach demonstrates promising performance, it does come with several limitations. First, current LLM inference engines do not offer special optimizations for tree search, meaning the proposed \treemodel still requires 2+ iterations, resulting in a performance that is approximately 2× slower than multi-chain sampling. Additionally, we utilize process supervision from tree search for RL training and attempt to optimize it from the perspective of advantage and re-weighting.
Further efforts, including how to assign appropriate weights to the importance of different steps, how to define more meaningful process signals from the tree structure, and how to implement step-level reward normalization, deserve more exploration.

\vpara{Acknowledgment.} This work has been in part supported by the Natural Science Foundation of China (NSFC) 62495063, Tsinghua University (Department of Computer Science and Technology) - Siemens Ltd., China Joint Research Center for Industrial Intelligence and Internet of Things (JCIIOT), and the New Cornerstone Science Foundation through the XPLORER PRIZE.
The corresponding author: Yuxiao Dong (yuxiaod@tsinghua.edu.cn).




\bibliography{reference}

\appendix


\newpage

\section{Position Distribution Analysis}

We study the position of selected forking tokens in their branches. As illustrated in \Cref{fig:forked_token_position_ratio_histogram}, we plot the relative positions of forking points, calculated as the ratio between a token's forking position and its branch length. The resulting distribution shows a roughly uniform pattern, which supports our assumptions presented in \Cref{sec:appendix:proof}.

\begin{figure}[htbp]
    \centering
    \includegraphics[width=0.9\linewidth]{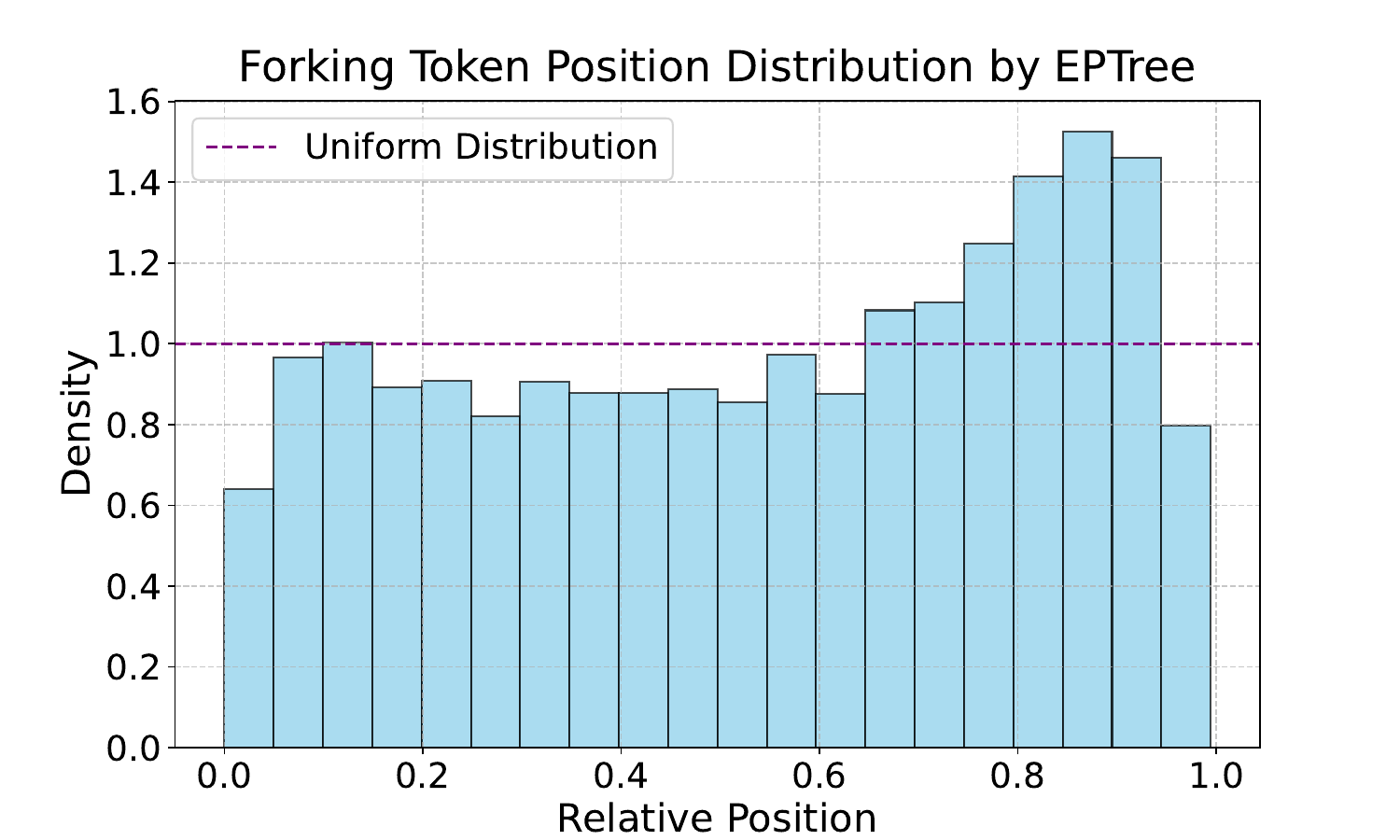}
    \caption{Distribution of the relative positions of forking tokens selected by \treemodel on Omni-MATH-500, based on the ratio of forking position to branch length.}
    \label{fig:forked_token_position_ratio_histogram}
\end{figure}

\section{Theoritical Analysis of \treemodel}
\label{sec:appendix:proof}

\begin{theorem}
\label{thm:leaf_num_bound}
Consider a setting where forking tokens follow a uniform distribution $\mathcal{U}(0,1)$ within each branch, and generation lengths remain fixed without repetition, subject to the constraint that parameter $l \leq 2$. Under these conditions, the entropy-guided tree search algorithm presented in \Cref{sec:entropy-tree} yields a leaf count that is bounded between $\dfrac{4}{3}$ and $\dfrac{12}{5}$ times that of a multi-sampling approach, while maintaining identical computational complexity in terms of total token evaluations.
\end{theorem}

\begin{proof}
We present the proof of \Cref{thm:leaf_num_bound} by analyzing two cases based on the branching process illustrated in \Cref{fig:branching}. Let $x_1, x_2, \dots, x_n$ be i.i.d. random variables uniformly distributed on $[0,1]$.

\begin{figure}[htbp]
    \centering
    \includegraphics[width=0.95\linewidth]{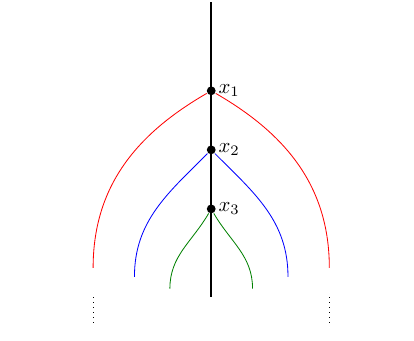}
    \caption{Forking token illustration when $n=3, t=2$, where $x_1, x_2, x_3 \stackrel{\text{i.i.d}}{\sim} \mathcal{U}(0,1)$.}
    \label{fig:branching}
\end{figure}

\textbf{Case $l=1$}: Assume all completion lengths are unit $1$. The total completion length of the entropy-tree is:
\[
L = 1 + t \cdot \sum_{i=1}^n (1 - x_i).
\]
The expected total length is:
\[
\mathbb{E}[L] = 1 + t \cdot \sum_{i=1}^n \mathbb{E}[1 - x_i] = 1 + \frac{nt}{2}.
\]
The number of leaves in the tree is:
\[
N_{\text{tree}} = 1 + nt.
\]
For the multi-sample approach with the same completion tokens, the number of leaves is:
\[
N_{\text{multi}} = 1 + \frac{nt}{2}.
\]
The ratio of the number of leaves in the tree to the multi-sample approach is:
\[
R(n, t) = \frac{1 + nt}{1 + \frac{nt}{2}}.
\]
Let $x = nt \geq 1$. Then:
\[
R(x) = \frac{1 + x}{1 + \frac{x}{2}}.
\]
The derivative of $R(x)$ with respect to $x$ is:
\[
R'(x) = \frac{1}{\left(1 + \dfrac{x}{2}\right)^2} > 0,
\]
indicating that $R(x)$ is monotonically increasing. Evaluating at the endpoints:
\[
R(1) = \frac{4}{3}, \quad \lim_{x \to \infty} R(x) = 2.
\]
Thus, $R(n, t) \in \left[\dfrac{4}{3}, 2\right)$ when $l=1$.

\textbf{Case $l=2$}: In this scenario, the next top-$n$ entropy token appears either on the main chain or on one of the branches. The probabilities of these events are:
\[
\frac{1}{1+t \sum_i\left(1-x_i\right)} \quad \text{and} \quad \frac{1-x_i}{1+t \sum_i\left(1-x_i\right)},
\]
Respectively. The expected new completion length is:
\[
\mathbb{E}\left[l_i\right] = \mathbb{E}\left[\frac{\dfrac{1}{2}+\dfrac{t}{2} \sum_i\left(1-x_i\right)^2}{1+t \sum_i\left(1-x_i\right)}\right] := \varphi.
\]
Through Monte Carlo simulation, it is observed that $\varphi$ monotonically decreases with $n$. As $n \rightarrow \infty$, the ratio tends to $\dfrac{1}{3}$. For $n=1$, the integral value is $\ln 2 - \dfrac{1}{4}$. Thus, $\varphi \in \left(\dfrac{1}{3}, \ln 2 - \dfrac{1}{4}\right]$.

The total generation length is:
\[
1 + \frac{1}{2} n t + n t \cdot \varphi,
\]
And the number of leaves is:
\[
1 + 2 n t.
\]
For the same completion length, the number of leaves in the multi-sample approach is:
\[
1 + \left(\frac{1}{2} + \varphi\right) n t.
\]
The ratio of the number of leaves in the tree to the multi-sample approach is:
\[
R(nt) = \frac{1 + 2 n t}{1 + \left(\dfrac{1}{2} + \varphi\right) n t},
\]
Which is also monotonically increasing in $nt$. Therefore:
\[
R(nt) \in \left[\dfrac{6}{3 + 2\varphi}, \dfrac{4}{1 + 2\varphi}\right).
\]
Specifically:
\[
\begin{cases}
    R(nt)\geq \dfrac{6}{3+2\cdot \left(\ln 2-\dfrac{1}{4}\right)}\approx 1.544 
    \\ R(nt)<\dfrac{4}{1+2\cdot \dfrac{1}{3}}=2.4
\end{cases}
\]
Combining both cases, the ratio $R(n, t)$ lies in the interval $\left[\dfrac{4}{3}, \dfrac{12}{5}\right)$.

\end{proof}

\section{Detailed Evaluation of \treemodel}
\Cref{tab:tree-parameter-evaluation} presents a comprehensive evaluation of \treemodel in comparison to the multi-chain baseline for $k=16$. We select the RL training configuration $(6,2,1,2)$, as it offers a similar inference cost to the chain-16 setting while demonstrating improved efficiency (requiring only one interaction) and effectiveness (achieving 56.9 compared to 52.4).
\label{sec:appendix:eptree_table}
\begin{table}[!t]
    \centering
    \caption{$(M,N,L,T)$ Parameter Evaluation of \treemodel on Omni-MATH-500. The table reports the leaf number, PassRate, and total generation tokens under different parameter combinations for $k=16$ and $k=64$.}
    \label{tab:tree-parameter-evaluation}
    \resizebox{0.45\textwidth}{!}{%
        \begin{tabular}{@{}cccccc@{}}
            \toprule[1.2pt]
            \multicolumn{1}{c}{$\textbf{(M, N, L, T)}$} & \textbf{\#Leaf} $\uparrow$ & \textbf{PassRate} $\uparrow$ & \textbf{\#Token} $\downarrow$ \\
            \midrule
            \multicolumn{4}{c}{\textit{$k=16$}} \\
            \midrule
            multi-16      & 16 & 52.4 & 19858 \\
            (8, 3, 1, 1)  & 32 & 58.4 & 25223 \\
            (7, 2, 1, 2)  & 35 & 59.2 & 26200 \\
            (6, 2, 2, 1)  & 30 & 54.4 & 20537 \\
            (6, 2, 1, 2)  & 30 & 56.9 & 22268 \\
            (5, 3, 1, 2)  & 35 & 58.6 & 25269 \\
            (5, 2, 1, 3)  & 35 & 56.0 & 25210 \\
            (5, 3, 2, 1)  & 35 & 58.0 & 22668 \\
            (5, 1, 2, 3)  & 35 & 58.6 & 21895 \\
            \midrule
            \multicolumn{4}{c}{\textit{$k=64$}} \\
            \midrule
            multi-64      & 64 & 67.4 & 79367 \\
            (16, 2, 2, 2) & 144 & 70.0 & 88257 \\
            (16, 4, 1, 2) & 144 & 71.4 & 101744 \\\
            (16, 2, 1, 4) & 144 & 72.6 & 100468 \\
            (9, 5, 1, 3)  & 144 & 71.9 & 98193 \\
            (9, 3, 1, 5)  & 144 & 71.6 & 97193 \\
            (8, 8, 1, 2)  & 136 & 70.7 & 92946 \\
            (8, 4, 1, 4)  & 136 & 69.7 & 90015 \\
            (8, 8, 2, 1)  & 136 & 68.6 & 79582 \\
            (8, 4, 2, 2)  & 136 & 71.0 & 77768 \\
            (8, 2, 2, 4)  & 136 & 69.4 & 76750 \\
            \bottomrule[1.2pt]
        \end{tabular}
    }
\end{table}

\section{Results on Other Reasoning Benchmarks}
Figure \ref{fig:qwen_results_appendix} and  \ref{fig:glm_results_appendix} provide additional RL training results on MATH500, AMC, LiveCodeBench, and AIME2024, demonstrating that \model with \treemodel sampling outperforms RL with traditional multi-chain sampling across various benchmarks. 
\label{sec:appendix:other_benchmark}
\begin{figure*}[htbp]
    \centering
        \begin{minipage}{\textwidth}
            \centering
            \begin{minipage}{0.4\textwidth}
                \centering
                \includegraphics[width=\textwidth]{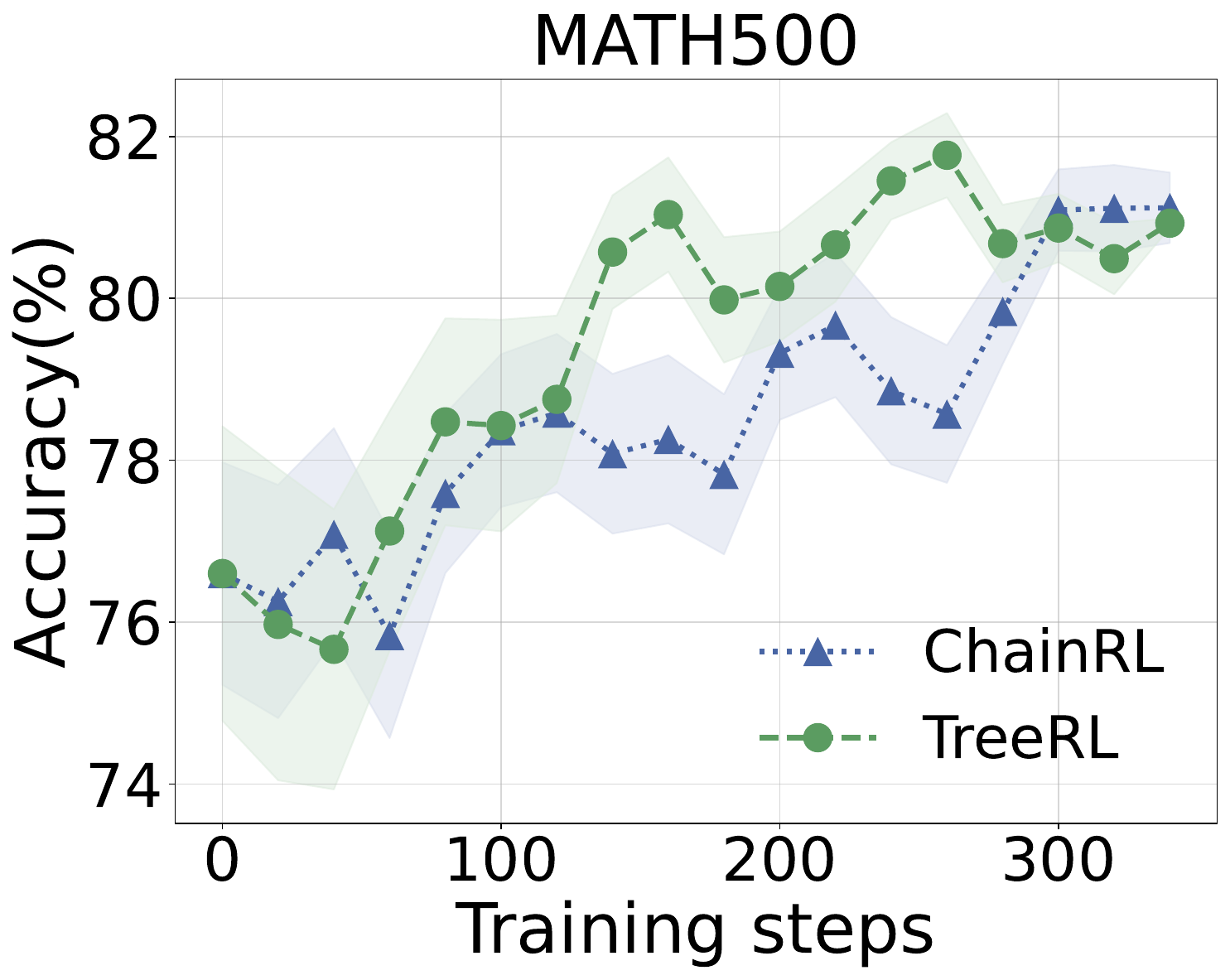}
                \vspace{-5mm}
                \label{fig:first_image_app}
            \end{minipage}
            \hfill
            \begin{minipage}{0.4\textwidth}
                \centering
                \includegraphics[width=\textwidth]{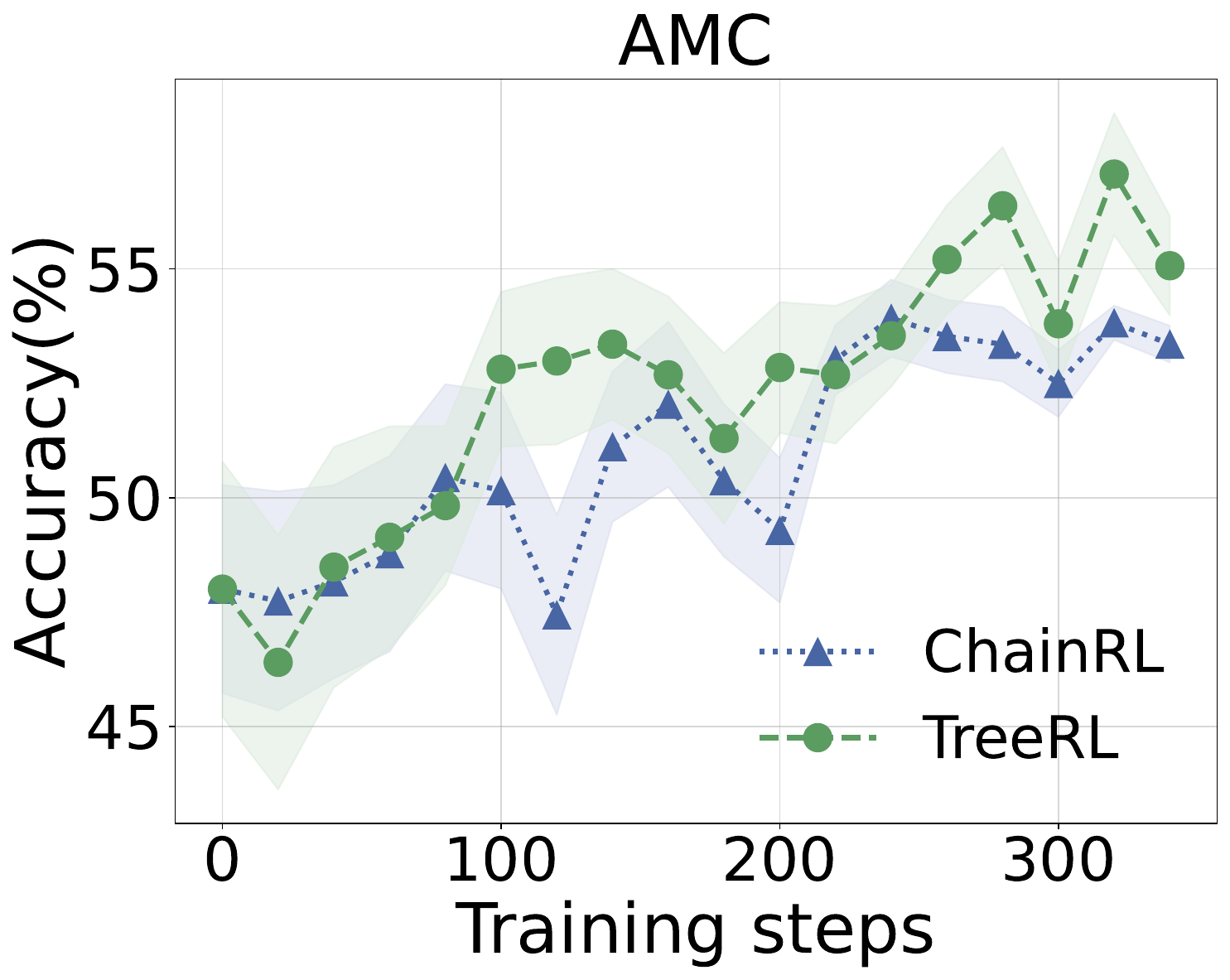}
                \vspace{-5mm}
                \label{fig:second_image_app}
            \end{minipage}
            \vspace{1mm}
            \\
            \begin{minipage}{0.4\textwidth}
                \centering
                \includegraphics[width=\textwidth]{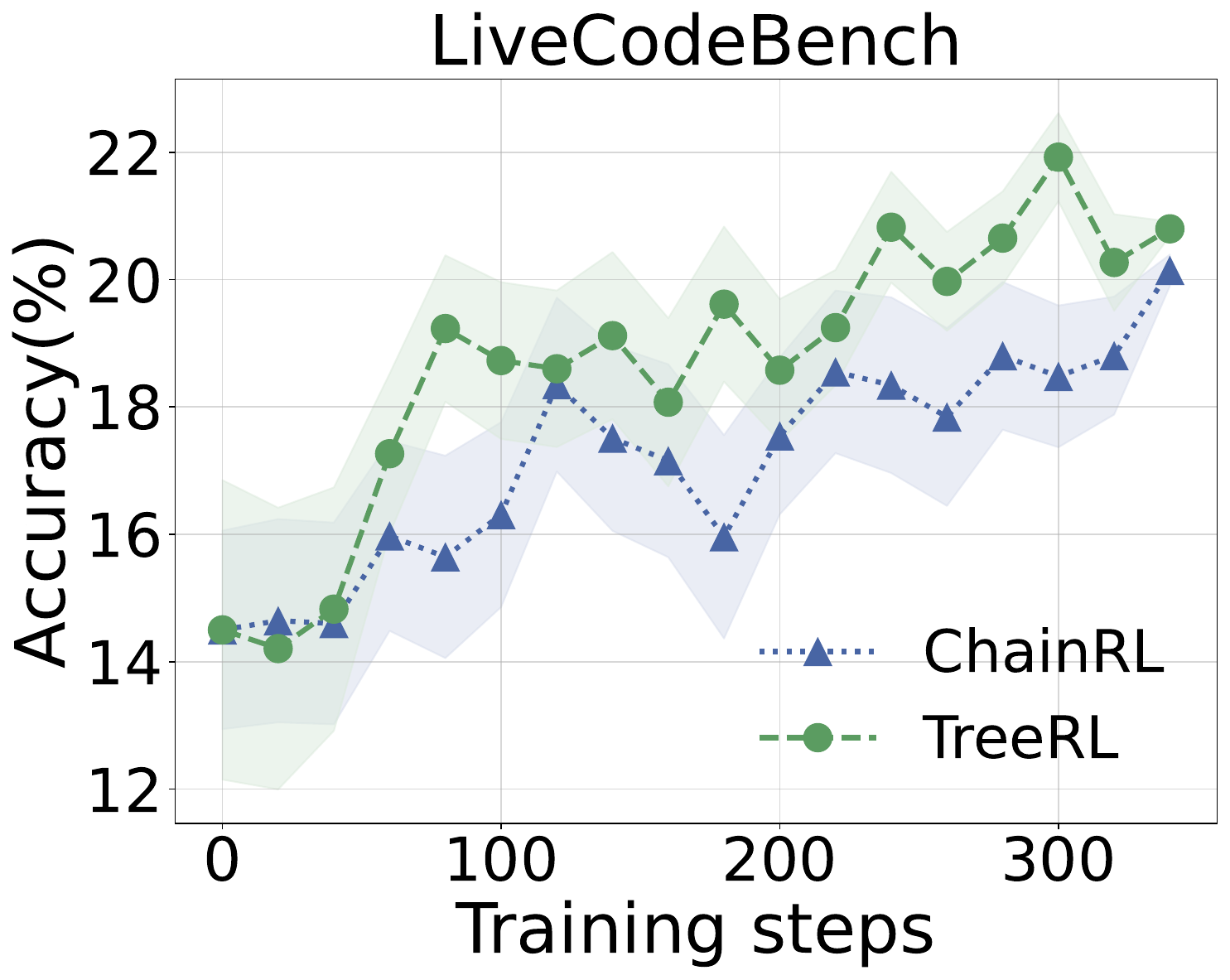}
                \vspace{-5mm}
                \label{fig:third_image_app}
            \end{minipage}
            \hfill
            \begin{minipage}{0.4\textwidth}
                \centering
                \includegraphics[width=\textwidth]{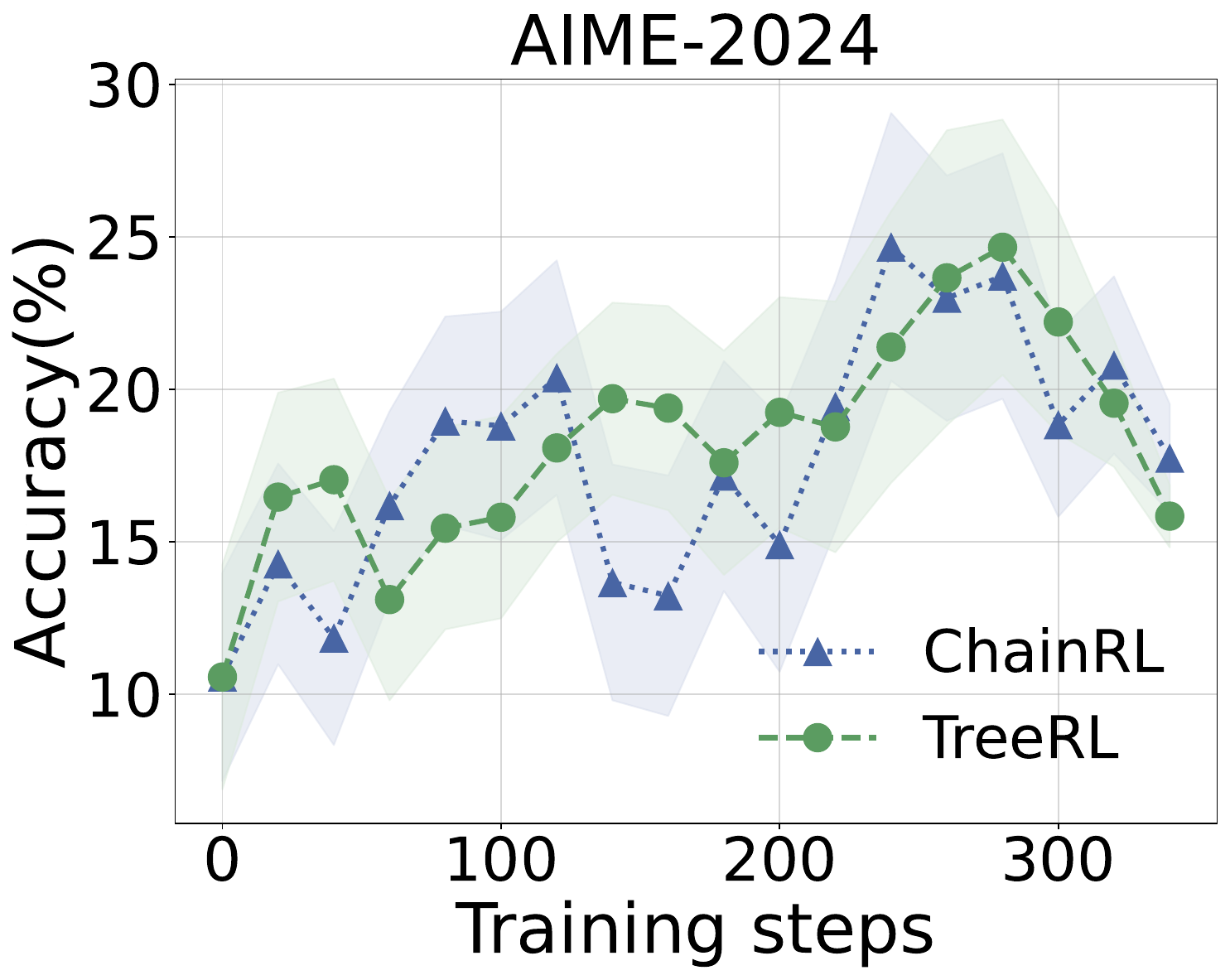}
                \vspace{-5mm}
                \label{fig:fourth_image_app}
            \end{minipage}
        \end{minipage}
    \caption{Performance comparison between \model and ChainRL on MATH500 (Upper Left), AMC (Upper Right), LiveCodeBench (Lower Left), and AIME2024 (Lower Right); experiments are based on Qwen-2.5-14B.}
    \label{fig:qwen_results_appendix}
\end{figure*}

\vspace{1mm}

\begin{figure*}[htbp]
    \centering
        \begin{minipage}{\textwidth}
            \centering
            \begin{minipage}{0.4\textwidth}
                \centering
                \includegraphics[width=\textwidth]{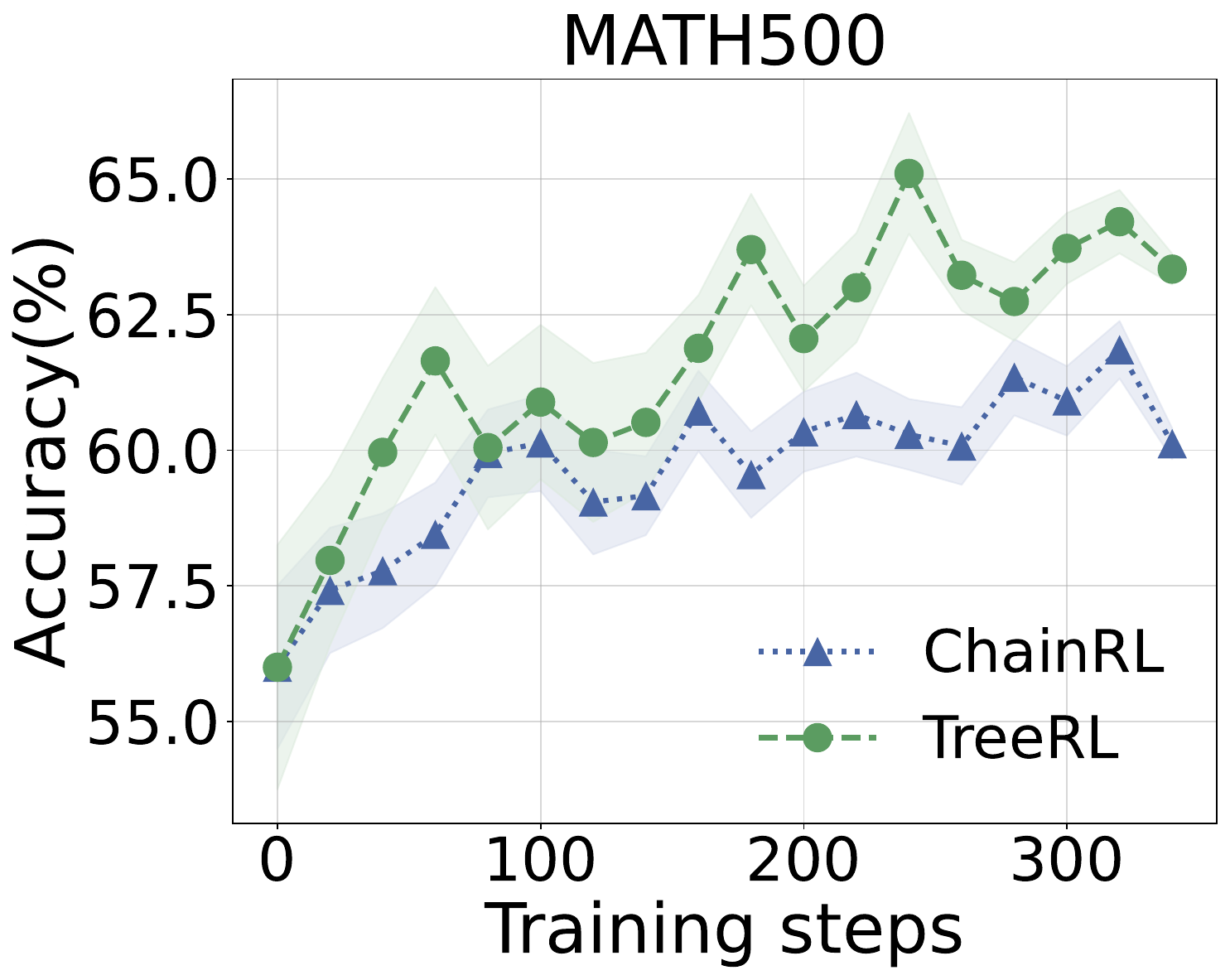}
                \vspace{-5mm}
                \label{fig:fifth_image}
            \end{minipage}
            \hfill
            \begin{minipage}{0.4\textwidth}
                \centering
                \includegraphics[width=\textwidth]{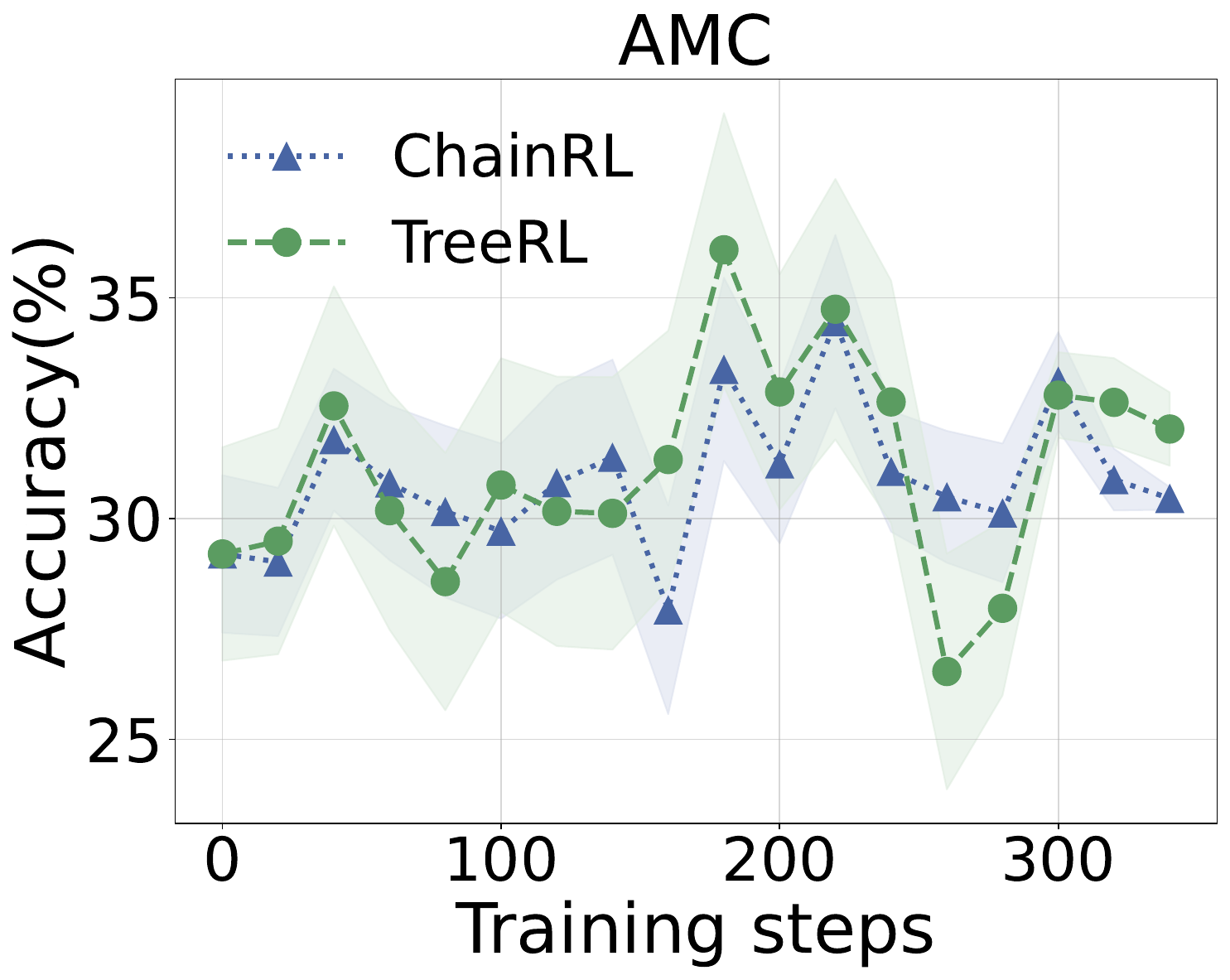}
                \vspace{-5mm}
                \label{fig:sixth_image_app}
            \end{minipage}
            \vspace{1mm}
            \\
            \begin{minipage}{0.4\textwidth}
                \centering
                \includegraphics[width=\textwidth]{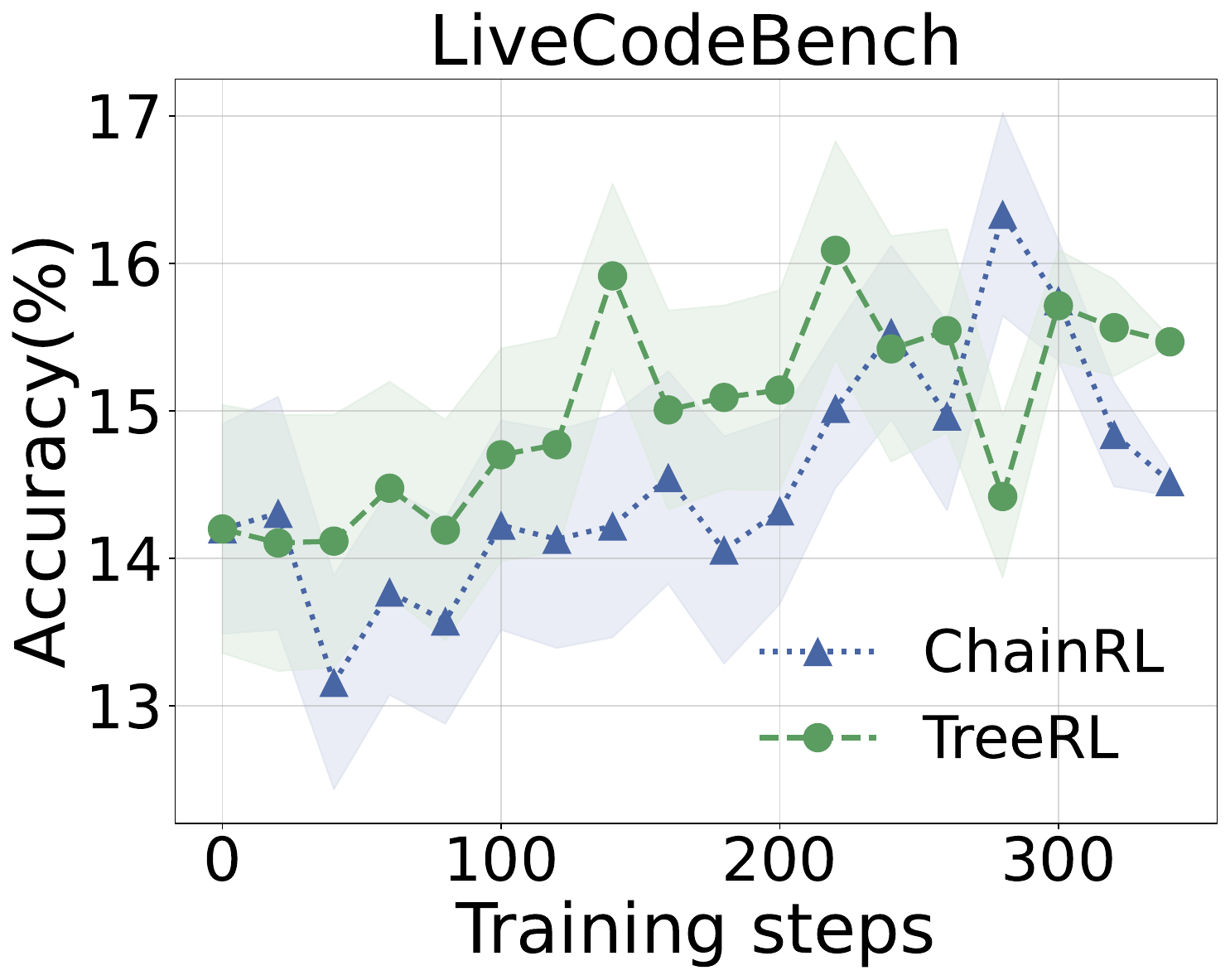}
                \vspace{-5mm}
                \label{fig:seventh_image}
            \end{minipage}
            \hfill
            \begin{minipage}{0.4\textwidth}
                \centering
                \includegraphics[width=\textwidth]{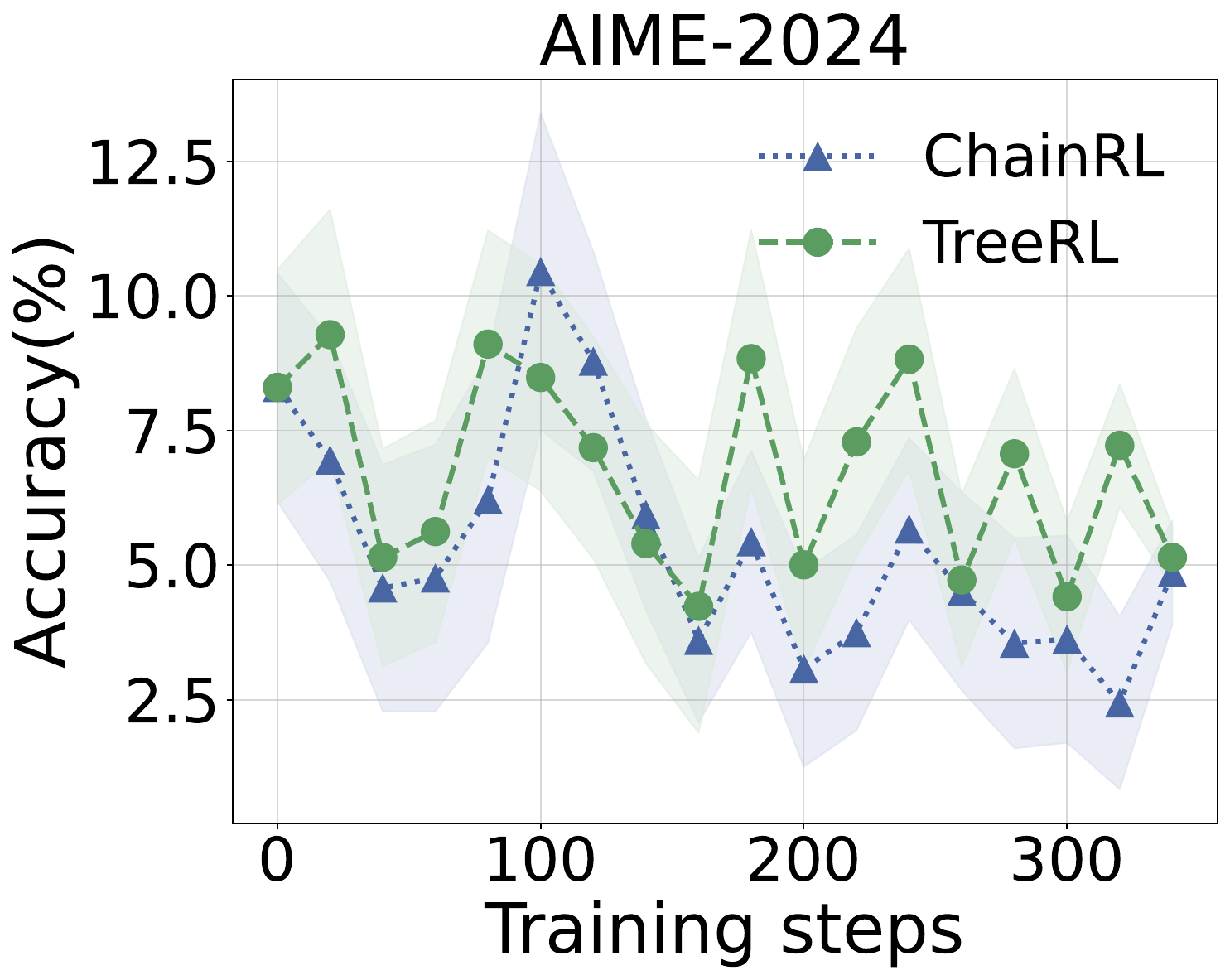}
                \vspace{-5mm}
                \label{fig:eighth_image}
            \end{minipage}
        \end{minipage}
    \caption{Performance comparison between \model and ChainRL on MATH500 (Upper Left), AMC (Upper Right), LiveCodeBench (Lower Left), and AIME2024 (Lower Right); experiments are based on GLM4-9B.}
    \label{fig:glm_results_appendix}
\end{figure*}

\section{\model on General tasks}

To further evaluate \model's generalizability, we further conducted experiments on general tasks.
We evaluate \model and ChainRL on $3$ general tasks: MMLU-Pro~\citep{wang2024mmlu}, Arena-Hard~\citep{li2024crowdsourced}, and IFEval~\citep{zhou2023instruction}.

\begin{itemize}
    \item MMLU-Pro extends MMLU~\citep{hendrycks2020measuring} with more challenging reasoning tasks, fewer noisy questions, and a larger answer set (4-10 options). We use the chain-of-thought prompt and measure the pass rate.
    \item Arena-Hard consists of 500 difficult prompts from Chatbot Arena, focusing on human-like preferences. We evaluate using the win rate score, comparing against the GPT-4-0314 baseline.
    \item IFEval tests the model's ability to follow prompt instructions. We use the strict prompt metric for evaluation.
\end{itemize}

\begin{table}[ht]
    \centering
    \resizebox{\linewidth}{!}{
        \begin{tabular}{lcccc}
        \toprule[1.2pt]
             &  MMLU-Pro & Arena-hard & IFEval & Avg\\
        \midrule
            SFT     & 57.6 & 57.5 & 49.9 & 55.0 \\
            ChainRL & 64.5 & 72.2 & 56.0 & 64.2 \\
            \model  & 64.5 & 71.8 & 58.2 & 64.9 \\
        \bottomrule[1.2pt]
        \end{tabular}
    }
    \caption{Performance on general benchmarks.}
    \label{tab:general-task-performance}
\end{table}


As is shown in \Cref{tab:general-task-performance}, we observe a comparable performance between \model and ChainRL. This indicates that while \model exhibits a particular advantage in reasoning tasks, it maintains its performance in general tasks, achieving robust performance in diverse task types.

\end{document}